\documentclass{article}

\usepackage[preprint]{neurips_2026}

\usepackage[utf8]{inputenc}
\usepackage[T1]{fontenc}
\usepackage{hyperref}
\usepackage{url}
\usepackage{booktabs}
\usepackage{amsfonts}
\usepackage{amsmath}
\usepackage{amssymb}
\usepackage{mathtools}
\usepackage{amsthm}
\usepackage{nicefrac}
\usepackage{microtype}
\usepackage{xcolor}
\usepackage{graphicx}
\usepackage{subcaption}
\usepackage{enumitem}
\usepackage{algorithm}
\usepackage{wrapfig}
\usepackage{algorithmic}
\usepackage{tikz}
\usetikzlibrary{arrows.meta,decorations.pathreplacing,calc,positioning,shapes.geometric}
\usepackage{thmtools}
\usepackage{xcolor}

\theoremstyle{plain}
\newtheorem{theorem}{Theorem}[section]
\newtheorem{proposition}[theorem]{Proposition}
\newtheorem{lemma}[theorem]{Lemma}

\theoremstyle{definition}
\newtheorem{definition}[theorem]{Definition}
\newtheorem{assumption}[theorem]{Assumption}
\theoremstyle{remark}

\newcommand{\R}{\mathbb{R}}
\newcommand{\E}{\mathbb{E}}

\newcommand{\cR}{\mathcal{R}}

\newcommand{\fnorm}[1]{\|#1\|_F}
\newcommand{\opnorm}[1]{\|#1\|_{\mathrm{op}}}
\newcommand{\nucnorm}[1]{\|#1\|_*}
\newcommand{\kfnorm}[2]{\|#1\|_{\mathrm{KF},#2}}
\newcommand{\dualnorm}[1]{\|#1\|_{(r)}}

\DeclareMathOperator{\col}{col}
\DeclareMathOperator{\orth}{orth}
\DeclareMathOperator{\ColNorm}{ColNorm}
\DeclareMathOperator{\diag}{diag}
\DeclareMathOperator{\erank}{erank}
\DeclareMathOperator{\clip}{clip}

\newcommand{\ip}[2]{\langle #1, #2 \rangle}

\DeclareMathOperator*{\argmax}{arg\,max}

\title{Orth-Dion: Eliminating Geometric Mismatch in Distributed Low-Rank Spectral Optimization}

\author{%
{\small
\begin{tabular}{@{}c@{\hspace{0.4em}}c@{\hspace{0.4em}}c@{\hspace{0.4em}}c@{}}
Tatsuhiro Nakamori\textsuperscript{1}\thanks{Equal contribution: \url{tatsuhironm@keio.jp, lpgomez@stanford.edu}}&
Laura Gomezjurado Gonzalez\textsuperscript{2*}&
Ganesh Talluri\textsuperscript{3} &
Ansh Tiwari\textsuperscript{4}
\\
Hideyuki Kawashima\textsuperscript{1}&
Ioannis Mitliagkas\textsuperscript{5,6}&
Guillaume Rabusseau\textsuperscript{5,6}&
Hiroki Naganuma\textsuperscript{5,6}\thanks{Corresponding author: \url{hiroki11x@gmail.com}}
\\[0.55em]
\multicolumn{4}{c}{\footnotesize
\textsuperscript{1}Keio University
\quad
\textsuperscript{2}Stanford University
\quad
\textsuperscript{3}Midwestern University
\quad
\textsuperscript{4}California Institute of Technology}
\\
\multicolumn{4}{c}{\footnotesize
\textsuperscript{5}Mila
\quad
\textsuperscript{6}Universit\'{e} de Montr\'{e}al}
\end{tabular}
}
}

\begin{document}

\maketitle


\begin{abstract}
Low-rank gradient compression reduces communication in distributed training by representing updates with rank-$r$ factors. Dion is a recent method that approximates Muon, a spectral optimizer that orthogonalizes momentum, using one step of power iteration followed by column normalization (rescaling each column of the right factor to unit length). This makes it compatible with fully sharded data parallel training, but it converges more slowly than full-rank spectral methods. We show that this gap is geometric: column normalization does not yield the rank-$r$ polar factor that Muon implicitly targets, so the resulting direction violates the dual-norm constraint of the low-rank spectral geometry, and the rate picks up an extra factor of $\sqrt{r}$ even though the low-rank approximation of the gradient itself is accurate. The same mismatch enters the smoothness term and the error-feedback recursion in the analysis, which has a knock-on effect on empirical performance. We propose Orth-Dion, which replaces column normalization with QR orthogonalization of the right factor. Under non-Euclidean smoothness, with $L_r$ the curvature constant along rank-$r$ directions, Orth-Dion attains rate $O(\sqrt{L_r/T})$, matching exact spectral methods at the same per-step communication cost as Dion. The proof removes the bounded-drift assumption common in prior error-feedback analyses via a self-consistent fixed-point argument, and uses a time-averaged contraction that only requires the error sequence to contract on average rather than at every step. Experiments on large-scale language model pre-training validate the predicted $\sqrt{r}$ scaling and show that Orth-Dion closes the convergence gap to Muon at Dion's communication cost.
\end{abstract}









\section{Introduction}
\label{sec:intro}

Training large language models increasingly relies on distributed sharding schemes such as fully sharded data parallelism~\citep{zhao2023pytorch}, where optimizer design is constrained by both per-step computation and communication. Spectral optimizers such as Muon~\citep{jordan2024muon} are attractive in this setting because they exploit the matrix structure of neural network parameters. Rather than applying an entrywise update, they orthogonalize momentum and take a step in a spectrally normalized direction. This geometry has been empirically effective, but it is difficult to scale under sharding, since forming and orthogonalizing full momentum matrices requires additional collective communication. Dion~\citep{ahn2025dion} addresses this bottleneck by replacing the full spectral update with a low-rank factorized one. Each matrix update is represented through rank-$r$ factors, which are computed by one step of power iteration and communicated at cost proportional to $(m+n)r$ rather than $mn$.

At first glance, the remaining gap between Dion and full-rank spectral methods appears to be the unavoidable price of using rank-$r$ updates~\citep{ahn2025dion, carlson2015preconditioned}.
Since Dion communicates only low-rank factors, one might expect its slower convergence to come from missing gradient directions outside the selected subspace.
This interpretation is natural given the low-rank, power-iteration-based compression lineage of PowerSGD and Dion~\citep{vogels2019powersgd,ahn2025dion}. Under this view, the problem is where the update points, and Dion loses because its rank-$r$ subspace is too small.
We show that this explanation is incomplete. Dion's power iteration can identify a useful low-rank subspace, yet the subsequent update can still be shaped incorrectly inside that subspace.
The source is Dion's final normalization step. After power iteration, Dion normalizes the columns of the right factor independently~\citep{ahn2025dion}. This preserves the span of the factor, but it does not produce the orthogonal right factor associated with a rank-$r$ spectral update~\citep{carlson2015preconditioned,bernstein2024oldoptimizer}. As a result, Dion can move in essentially the same subspace as the desired low-rank polar direction while using a direction that is incorrectly scaled for spectral descent. We formalize this mismatch through a dual-norm factor $\nu_t$, showing that column normalization can introduce a rank-dependent penalty even when the low-rank approximation itself is accurate.

Orth-Dion fixes this mismatch with the smallest possible change. It keeps Dion's power iteration, residual buffer, and low-rank communication pattern unchanged, but replaces column normalization with QR orthogonalization of the right factor. This makes the update geometrically aligned with the rank-$r$ polar direction targeted by spectral methods. Under the same low-rank communication cost as Dion, Orth-Dion removes the rank-dependent normalization penalty and recovers the leading convergence rate of exact rank-$r$ spectral updates. Because Dion-family methods use error feedback, the proof must control the coupled evolution of the residual buffer and the tracked subspace. We do this through a self-consistent residual-control argument and an amortized contraction condition, avoiding the bounded-buffer-drift assumption used in prior analyses.

We measure $\nu_t$ directly during Llama 3 320M pretraining. Dion's $\nu_t$ rises with rank across every layer and logged step; Orth-Dion's holds at $\nu_t \approx 1$. The mismatch is not a worst-case anomaly; it is what Dion does in the regime it is meant for. Holding $r$ fixed and swapping column normalization for QR lowers validation loss at every rank we tested, which separates the QR fix from any rank-capacity effect. Because QR adds per-step compute, we pair it with adaptive rank to obtain Ada-Orth-Dion: the working rank shrinks toward each layer's intrinsic dimensionality, recovering Dion-like step time on a 17.1B model (Figure~\ref{fig:fig1}) while maintaining convergence gain. The combined method gives a low-rank distributed spectral optimizer with a better convergence--communication trade-off than Dion.






    

\begin{figure*}[t]
  \centering
  \begin{subfigure}[b]{0.41\textwidth}
    \centering
    \includegraphics[width=\linewidth]{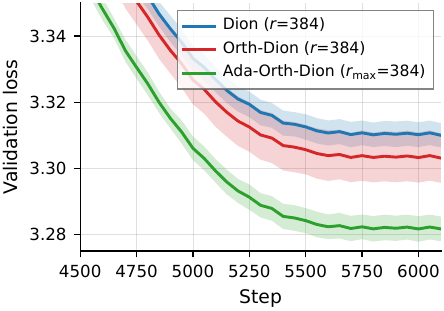}
    \caption{\footnotesize \textbf{Lower validation loss at matched rank.}
      LLaMA 320M, late-stage convergence (mean$\pm$2 std, full trajectory in Fig.~\ref{fig:main_4shards}).}
    \label{fig:fig1_val_loss}
  \end{subfigure}\hfill
  \begin{subfigure}[b]{0.57\textwidth}
    \centering
    \includegraphics[width=\linewidth]{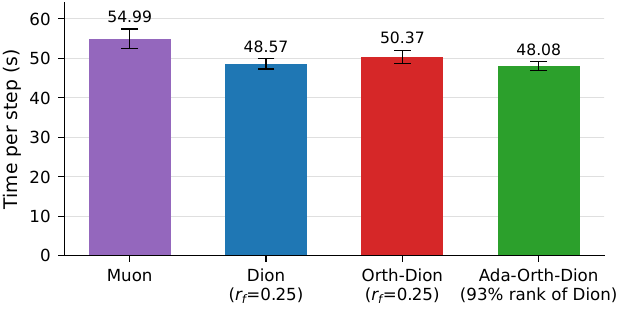}
    \caption{\footnotesize \textbf{Faster wallclock at scale.}
      Per-step time on a LLaMA 17.1B 50-step run (mean$\pm$2 std). Rank fraction $r_f$ is the spectral update's rank divided by the parameter matrix's full rank.}
    \label{fig:fig1_wallclock}
  \end{subfigure}
  \caption{\footnotesize \textbf{Proposed methods improve both convergence and wallclock.} \textbf{(a)}~At matched rank, Orth-Dion and Ada-Orth-Dion reach Dion's plateau earlier and keep going. \textbf{(b)}~Adaptive rank absorbs Orth-Dion's QR overhead and matches Dion's per-step time on the 17.1B model. (Ada-Orth-Dion's rank is pinned at $0.93\,r_f$, the steady-state rank reached on the 320M model; see App.~\ref{app:llm_details:wide}.)}
  \label{fig:fig1}
\end{figure*}

\paragraph{Contributions.}
We (i)~identify a geometric source of suboptimality in Dion: column normalization inflates the Ky Fan dual norm by up to $\sqrt{r}$, producing a rank-dependent convergence penalty distinct from low-rank approximation error; (ii)~propose \textbf{Orth-Dion}, a one-line ColNorm$\to$QR replacement that forces $\nu_t{=}1$ and recovers the exact low-rank spectral rate at Dion's communication cost; (iii)~prove convergence under non-Euclidean smoothness via a self-consistent residual-control argument that eliminates bounded-buffer-drift and allows amortized rather than per-step contraction; and (iv)~validate the mechanism and optimizer-level impact through direct $\nu_t$ measurements, matched-rank improvements on Llama~3 320M, and Ada-Orth-Dion matching Dion's wall-clock time at lower validation loss.

\section{Background: Geometry of Rank-Constrained Optimization}
\label{sec:geometry}

\textbf{Steepest descent under norms.}\;
Given $f:\R^{m\times n}\to\R$, steepest descent under norm $\|\cdot\|$ updates $X_{t+1} = X_t - \eta D_t^*$ with $D_t^* = \argmax_{\|D\|\leq 1}\ip{\nabla f(X_t)}{D}$.
SGD uses $\ell_2$; SignSGD~\citep{bernstein2018signsgd} uses $\ell_\infty$; spectral methods use the operator norm.
For distributed training, we constrain updates to rank $\leq r$, yielding the \emph{Ky Fan $r$-norm} geometry. The Ky Fan r-norm is the dual of the sum of the top-r singular values, and can be written as
$D_t^* = \argmax_{\dualnorm{D}\leq 1} \ip{M_t}{D}$,
where $\dualnorm{D} = \max\{\sigma_1(D),\, \fnorm{D}/\sqrt{r}\}$.
The solution is the rank-$r$ polar factor $P_r(M_t) = U_r V_r^\top$ from the truncated SVD, achieving $\ip{M_t}{D_t^*} = \kfnorm{M_t}{r} = \sum_{i=1}^r \sigma_i(M_t)$.

\textbf{The two algorithms.}\;
Both maintain buffer $M_t = G_t + R_t$ and use one power iteration step.
The \emph{sole difference} is the right-factor normalization:
\begin{center}
\vspace{-0.3em}
\small
\begin{tabular}{lll}
& \textbf{Stripped Dion}\footnotemark & \textbf{Orth-Dion (Proposed)} \\

Line 5: & $\bar{V}_t \leftarrow \ColNorm(W_t)$ & $\bar{V}_t \leftarrow \orth(W_t)$ \quad\textit{(QR)}
\end{tabular}
\vspace{-0.3em}
\end{center}
\footnotetext{\emph{Stripped Dion}: Dion's spectral update without the auxiliary scalar/Adam-side machinery (decoupled scalar LR, output-head LR scaling). Theorem~\ref{thm:orth_dion} uses this form, but the QR correction is spectral-side only and transfers to full Dion, which all experiments use as the baseline (auxiliary machinery enabled for both); reported $\nu_t$ (Fig.~\ref{fig:sqrt_r_scaling}) and matched-rank gains (Table~\ref{tab:llm_main}) are in this setting.}

where $W_t = M_t^\top U_t$ and $U_t = \orth(M_t V_{t-1})$.
The update is $\hat{D}_t = U_t \bar{V}_t^\top$, and the error feedback is $R_{t+1} = \beta\,(I - U_t U_t^\top)\, M_t$.
QR costs $O(nr^2)$ which is negligible compared to the $O(mnr)$ power step since $r \ll m$.
The full procedure is given in Algorithm~\ref{alg:orth_dion}.

\begin{algorithm}[t]
\caption{Orth-Dion / Stripped Dion (one step)}
\label{alg:orth_dion}
\begin{algorithmic}[1]
\REQUIRE Gradient $G_t \in \R^{m \times n}$, residual $R_t$, right factor $V_{t-1}$, step size $\eta$, EF coefficient $\beta$
\STATE $M_t \leftarrow G_t + R_t$ \hfill $\triangleright$ Buffer = gradient + error feedback
\STATE $U_t \leftarrow \orth(M_t V_{t-1})$ \hfill $\triangleright$ Left factor via QR, $O(mr^2)$
\STATE $W_t \leftarrow M_t^\top U_t$ \hfill $\triangleright$ Right factor (un-normalized), $O(mnr)$
\STATE $\bar{V}_t \leftarrow \orth(W_t)$ \hfill $\triangleright$ \textbf{Orth-Dion}: QR; \textit{Dion}: $\ColNorm(W_t)$
\STATE $\hat{D}_t \leftarrow U_t \bar{V}_t^\top$ \hfill $\triangleright$ Low-rank update direction
\STATE $X_{t+1} \leftarrow X_t - \eta \hat{D}_t$ \hfill $\triangleright$ Parameter update
\STATE $R_{t+1} \leftarrow \beta\,(M_t - U_t (U_t^\top M_t))$ \hfill $\triangleright$ Error feedback
\STATE $V_t \leftarrow \bar{V}_t$ \hfill $\triangleright$ Warm-start for next step
\end{algorithmic}
\end{algorithm}

\textbf{FSDP communication.}\;
Under FSDP, parameters are sharded across devices.
Each training step involves: (1)~all-gather to materialize full weights, (2)~forward/backward pass, (3)~reduce-scatter to re-shard gradients, (4)~local optimizer step.
Muon requires an \emph{additional} all-gather/reduce-scatter pair for the full momentum matrix during orthogonalization.
Dion and Orth-Dion avoid this by communicating only the low-rank factors ($O((m{+}n)r)$ vs.\ $O(mn)$), making them compatible with FSDP sharding at rank-proportional cost.







\textbf{Assumptions.}\;
We use the following assumptions (details in Appendix~\ref{app:assumptions}):
\textbf{(A)}~\emph{Non-Euclidean smoothness}: $f(X{+}\Delta) \leq f(X) + \ip{\nabla f(X)}{\Delta} + \frac{L_r}{2}\dualnorm{\Delta}^2$, where $L_r$ captures curvature along low-rank directions;
\textbf{(A$'$)}~\emph{Frobenius gradient-Lipschitz}: $\fnorm{\nabla f(Y)-\nabla f(X)} \leq L_F\fnorm{Y-X}$ (independent of A in non-Euclidean geometry);
\textbf{(B$'$)}~\emph{Gradient spectral gap}: $\sigma_r(G_t) - \sigma_{r+1}(G_t) \geq \Delta_{\mathrm{gap}} > 0$;
\textbf{(B$''$)}~\emph{Small spectral tail}: $\sigma_{r+1}(G_t) \leq \tau$ with $\tau = O(\eta)$ on the relevant horizon;
\textbf{(C$'$)}~\emph{Gradient bounds}: $\fnorm{G_t} \leq G_F$, $\kappa_r(G_t) \leq \kappa_G$.

\textbf{One-step descent decomposition.}\;
Under Assumption~A, any update $\hat{D}_t$ with $\dualnorm{\hat{D}_t} = \nu_t$ yields (proof in Appendix~\ref{app:one_step}):
\begin{equation}
\label{eq:one_step}
f(X_{t+1}) \leq f(X_t) - \eta\kfnorm{G_t}{r} + \eta\big(\delta_t + (1{+}\nu_t)\kfnorm{R_t}{r}\big) + \tfrac{L_r \nu_t^2}{2}\eta^2,
\end{equation}
where $\delta_t = \kfnorm{M_t}{r} - \ip{M_t}{\hat{D}_t}$ is the oracle defect.
The dual norm factor $\nu_t$ appears in two critical places: the smoothness penalty $L_r\nu_t^2\eta^2/2$ and the residual coupling $(1{+}\nu_t)\kfnorm{R_t}{r}$. 


\section{\texorpdfstring{The $\sqrt{r}$ Inefficiency}{The sqrt(r) Inefficiency}}
\label{sec:sqrt_r}

Column normalization produces $\bar{V}_t$ with unit-length columns, giving Gram matrix $\bar{V}_t^\top\bar{V}_t$ with diagonal entries 1 and potentially large off-diagonal entries.
Since this is a correlation matrix:

\begin{proposition}[$\sqrt{r}$ bound]
\label{prop:colnorm_sqrt_r}
For column-normalized $\bar{V}_t$: $\nu_t = \dualnorm{\hat{D}_t} = \opnorm{\bar{V}_t} \in [1, \sqrt{r}]$.
Equality $\nu_t{=}1$ holds only when columns of $W_t$ are already orthogonal.
\end{proposition}

\textbf{How $\nu_t$ enters the rate.}\;
Telescoping~\eqref{eq:one_step} and optimizing $\eta$:
\begin{equation}
\label{eq:rate_with_nu}
\min_t \kfnorm{G_t}{r} \leq \sqrt{\tfrac{2(f_0{-}f_\infty)L_r \nu^2}{T}} + O(1/T), \quad \nu = \max_t \nu_t.
\end{equation}
With $\nu{=}\sqrt{r}$ (ColNorm): rate $= O(\sqrt{L_r r/T})$.
With $\nu{=}1$ (Orth-Dion): rate $= O(\sqrt{L_r/T})$.
The $\sqrt{r}$ gap is purely geometric.

\textbf{Geometric interpretation.}\;
The Ky Fan dual ball $\{D:\dualnorm{D}{\leq}1\}$ constrains both $\opnorm{D}{\leq}1$ and $\fnorm{D}{\leq}\sqrt{r}$.
The polar factor $P_r(M_t)$ lies on its boundary ($\opnorm{\cdot}{=}1$, $\fnorm{\cdot}{=}\sqrt{r}$).
ColNorm preserves $\fnorm{\cdot}{=}\sqrt{r}$ but allows $\opnorm{\cdot}$ up to $\sqrt{r}$---pushing the update \emph{outside} the dual ball and incurring a factor-$r$ smoothness penalty.
The operator norm constraint $\opnorm{\hat{D}_t}{\leq}1$ is the binding constraint for any rank-$r$ partial isometry.
ColNorm violates it by concentrating singular values non-uniformly ($\sigma_{\max}(\hat{D}_t) \gg \sigma_{\min}(\hat{D}_t)$); QR orthogonalization forces all singular values to 1 by construction.
The key insight: the inefficiency is not in \emph{which subspace} is selected (Lemma~\ref{lem:same_tracking} shows both algorithms select the same subspace) but in \emph{how the update is scaled within that subspace}.

\section{Orth-Dion: Correcting the Geometry}
\label{sec:orth_dion}

\begin{lemma}[Partial isometry, $\nu_t{=}1$]
\label{lem:partial_isometry}
If $U{\in}\R^{m\times r}$, $\bar{V}{\in}\R^{n\times r}$ both have orthonormal columns, then $\hat{D}{=}U\bar{V}^\top$ has $\dualnorm{\hat{D}}{=}1$.
\end{lemma}

\begin{lemma}[Non-negative defect]
\label{lem:nonneg_defect}
For Orth-Dion, $\delta_t \geq 0$ for all $t$.
\emph{(By H\"older: $\ip{M_t}{\hat{D}_t} \leq \kfnorm{M_t}{r}\cdot 1$.)}
\end{lemma}

\begin{lemma}[Same tracking]
\label{lem:same_tracking}
The projector error $\varepsilon_t$ satisfies the same recursion for both algorithms, since $\col(\orth(W)) = \col(\ColNorm(W))$ and the error feedback depends only on $U_t$.
\end{lemma}

\textbf{Consequence.}\;
The improvement is \emph{entirely} in $\nu_t$, not approximation quality.
Lemma~\ref{lem:same_tracking} ensures that both algorithms track the same gradient subspace, since the column spaces of $\orth(W)$ and $\ColNorm(W)$ are identical.
At typical ranks, the quantitative impact is substantial. For $r{=}256$, the smoothness cost ratio is $L_r\nu^2/L_r = r = 256\times$, while the residual coupling ratio is $(1{+}\sqrt{r})/(1{+}1) = (1{+}16)/2 = 8.5\times$.
These factors compound, inflating the leading convergence constant by $\sqrt{r}$ through both the smoothness and coupling terms.


\begin{figure*}[t]
  \centering
  \begin{subfigure}[t]{0.48\textwidth}
    \centering
    \includegraphics[width=\linewidth]{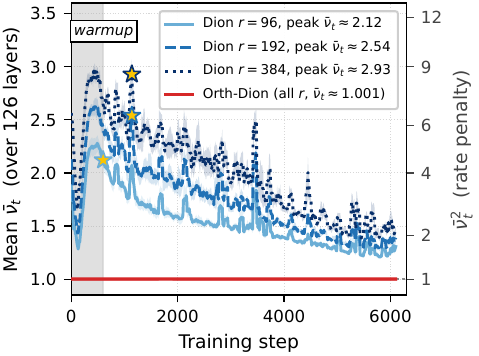}
    \caption{Layer-mean $\bar{\nu}_t$ vs.\ training step.}
    \label{fig:nu_bar}
  \end{subfigure}\hfill
  \begin{subfigure}[t]{0.48\textwidth}
    \centering
    \includegraphics[width=\linewidth]{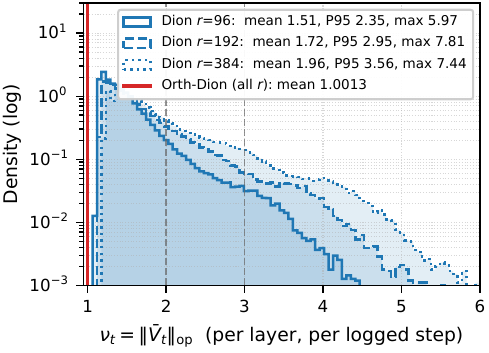}
    \caption{Post-warmup pooled $\nu_t$ distribution.}
    \label{fig:nu_hist}
  \end{subfigure}
  \caption{\footnotesize \textbf{The dual-norm mismatch is real, rank-ordered, and persistent} (Llama~3 320M, $r{\in}\{96,192,384\}$, mean over $3$ seeds; band $\pm$std). \textbf{(a)}~Dion's $\bar{\nu}_t$ separates cleanly by rank and stays well above $1$ throughout training; the right axis shows $\bar{\nu}_t^2$, the factor through which $\nu_t$ enters the smoothness term in \eqref{eq:one_step}. \textbf{(b)}~Per-update $\nu_t$ pooled over layer, step, and seed: the inflation is dispersed, not transient. Orth-Dion collapses to $\nu_t{\approx}1$ in both views (max $\approx 1.003$).}
  \label{fig:sqrt_r_scaling}
\end{figure*}

The bound $\nu_t = \sqrt{r}$ is attained only when all columns of $W_t$ are colinear. This pathological case does not arise in practice. But the inflation it predicts is real, rank-ordered, and persistent throughout training. Figure~\ref{fig:sqrt_r_scaling} measures $\nu_t$ during Llama~3 320M pretraining at $r\in\{96,192,384\}$. Figure~\ref{fig:nu_bar} shows the time-course: Dion's curves separate cleanly by rank (peaks $\approx\!2.1, 2.5, 2.9$, i.e.\ $\bar{\nu}_t^2 \approx 4.5{-}8.6$) and stay strictly above $1$ throughout. Figure~\ref{fig:nu_hist} shows the per-update dispersion: at $r{=}384$, $\sim\!33\%$ of individual updates carry $\nu_t > 2$ and $\sim\!10\%$ carry $\nu_t > 3$. In our runs $\nu_t$ never approaches the worst-case $\sqrt{r}$. It does climb with rank, though, and stays elevated through training. Orth-Dion's $\nu_t$ holds at $1$ everywhere we measured. No hyperparameter setting fixes this; the problem is in the normalization step itself.
Proofs of Lemmas~\ref{lem:partial_isometry}--\ref{lem:same_tracking} are in Appendix~\ref{app:orth_dion_proofs}.

\section{Convergence Theory}
\label{sec:theory}




\subsection{Self-Consistent Drift Elimination}
\label{sec:self_consistent}

The original Dion analysis assumes bounded buffer drift ($\opnorm{M_t{-}M_{t-1}} \leq \Delta_{\mathrm{drift}}$, Assumption~D), but this is circular: $M_t{-}M_{t-1}$ contains the residual increment, which is itself $O(\|G_t\|)$ rather than small. We eliminate this assumption by exploiting an asymmetry: under Assumption A$'$, the \emph{gradient} subspace changes slowly ($\sin\Theta_{\max}(P_G^{t+1}, P_G^t) \leq L_F\,\eta\sqrt{r}/\Delta_{\mathrm{gap}} = O(\eta)$), while the buffer subspace may change abruptly. We track the former and \emph{derive} the latter via a self-consistency loop on the coupled system $(\varepsilon_t, \cR_t)$, where $\varepsilon_t = \sin\Theta_{\max}(P_M^t, P_G^t)$ is the tracking error and $\cR_t = \fnorm{R_t}/\fnorm{G_t}$ the residual ratio. The recursions $\varepsilon_{t+1} \leq \tilde{\gamma}_{\mathrm{eff}}(1{+}\kappa_G)\varepsilon_t + O(\eta)$ and $\cR_{t+1} \leq h(\varepsilon_t,\cR_t)$ form a contractive map $\Phi$ on $[0,1]^2$; for $\eta = O(\Delta_{\mathrm{gap}}/(L_F\sqrt{r}\kappa_G))$, Banach fixed-point yields a unique fixed point with $\cR_\infty,\varepsilon_\infty = O(\eta)$ \emph{provided} the spectral tail is also small (Assumption B$''$); without B$''$ the conclusion is $\cR_\infty=O(\tau+\eta)$. The residual recursion is controlled by the spectral tail $\sigma_{r+1}(M_t)$ plus a tracking-leakage term $\varepsilon_t\sigma_1(M_t)$ (Appendix~\ref{app:self_consistent_proof}).





\subsection{Amortized Contraction}
\label{sec:amortized}

The per-step contraction $\rho_t = \tilde{\gamma}_{t}(1{+}\kappa_{G,t})$ may exceed $1$ during warmup or landscape transitions, so a uniform per-step requirement is too strong. We instead impose \emph{uniform amortized contraction}: there exist $\bar{\rho}\in(0,1)$, $C_\rho\geq 1$ such that
\[
\prod_{i=s}^{t-1} \rho_i \;\le\; C_\rho \,\bar{\rho}^{\,t-s}, \qquad 0 \leq s < t.
\]
Every suffix product decays geometrically up to a constant, allowing transient $\rho_t>1$ while ruling out persistent growth. Under this condition the tracking-error recursion $\varepsilon_t \leq \rho_{t-1}\varepsilon_{t-1} + c_{t-1}$ unrolls to geometric convergence up to a bounded bias term (Theorem~\ref{thm:amortized}, Appendix~\ref{app:amortized_proof})---the natural Lyapunov-style stability condition for subspace tracking, where stability is governed by contraction over windows rather than at every step.

\subsection{Main Theorem}

\begin{theorem}[Convergence of Orth-Dion]
\label{thm:orth_dion}
Under Assumptions A, B', C' with $\beta{=}1$, $\rho'{<}1$, and $\eta = c/\sqrt{T}$:
\begin{equation}
\min_{0 \leq t < T} \kfnorm{G_t}{r} \leq \sqrt{\frac{2(f_0 - f_\infty) L_r}{T}} + O\!\left(\frac{1}{T}\right)
\end{equation}
\vspace{-1em}
\end{theorem}

The leading constant matches exact-SVD Muon, where $\nu{=}1$, and improves on Stripped Dion's $\sqrt{2(f_0{-}f_\infty)L_r r/T}$, where $\nu{=}\sqrt{r}$, by a factor of $\sqrt{r}$.
The full per-step cost is $O(mnr) + O(nr^2)$, matching Stripped Dion.
The improvement is geometric, since Orth-Dion achieves $\nu{=}1$ by construction and eliminates the rank-dependent inflation in the smoothness term.
The proof appears in Appendix~\ref{app:main_proof}.

\section{\texorpdfstring{Error Feedback and the $\beta<1$ Regimes}{Error Feedback and the beta < 1 Regimes}}
\label{sec:beta_regimes}
The one-step descent bound \eqref{eq:one_step} contains two rank-dependent terms: the smoothness penalty $L_r\nu_t^2\eta^2/2$, addressed in Sections~\ref{sec:sqrt_r}--\ref{sec:orth_dion} via $\nu_t{=}1$, and the residual coupling $(1{+}\nu_t)\kfnorm{R_t}{r}$, whose dynamics depend on $\beta$. With fixed $\nu_t$, the residual recursion $R_{t+1} = \beta\,(I - U_t U_t^\top)(G_t{+}R_t)$ is governed by the cross-step correlation of the out-of-subspace gradient. Characterizing this correlation yields a $0.315$-nat improvement on GPT-2~45M at zero additional compute (Table~\ref{tab:beta_sweep}) and shows that EF is mechanistically distinct from heavy-ball momentum. 
Define the \emph{gradient subspace persistence}
$\phi_t := \cos\angle\big((I{-}P_t)G_t,\; (I{-}P_{t-1})G_{t-1}\big)$ and the \emph{R-norm ratio} $\cR_t := \fnorm{R_t}/\fnorm{G_t}$. Details in Appendix~\ref{app:beta_details}.

\textbf{Regime~1 (Coherent, $\phi_t {>} 0$):}\;
Out-of-subspace components persist across steps, so $R_T \approx \sum_t (I{-}P_t)G_t$ accumulates coherently.
The subspace eventually rotates to capture this direction; $\beta{=}1$ is optimal.

\textbf{Regime~2 (Stochastic, $\phi_t {\approx} 0$):}\;
Typical of LLM training with stochastic mini-batches.
The residual performs a random walk, acting as \emph{implicit heavy-ball momentum} with effective coefficient $\hat{\varepsilon}^2$, where $\hat{\varepsilon}_t = \|(I{-}P_t)M_t\|_F/\fnorm{M_t}$ is the tracking error.
Reducing $\beta$ below 1 boosts this effect at zero extra compute. The optimal $\beta$ satisfies $\beta^* \approx 1{-}2\hat{\varepsilon}$; for GPT-2 with baseline tracking error $\hat{\varepsilon}{\approx}0.35$ measured at $\beta{=}1$, this predicts $\beta^*\!\ll\!1$. Empirically, val loss decreases monotonically from $\beta{=}1$ down to $\beta{\approx}0.1$ (Table~\ref{tab:beta_sweep}), and $\cR_t^*$ grows from 0.3 to 3.2 over the same range, directly tracking the implicit-momentum accumulation the theory predicts. Explicit Polyak momentum is not a substitute: PolyakDion ($\mu{=}0.95$, $R{\equiv}0$) reaches val.\ loss $3.162$ vs.\ $3.056$ for Dion at $\beta{=}0.3$, since EF reshapes the buffer \emph{before} compression while explicit momentum acts in the full parameter space.

\textbf{Regime~3 (Anti-correlated, $\phi_t {<} 0$):}\;
Residual components cancel across steps; the buffer diverges and EF is harmful.
$\beta{\to}0$ is necessary.

This framework extends EF21~\citep{richtarik2021ef21}, which assumes independent lost signal (Regime~2 only) and misses the coherent and anti-correlated cases.
Full derivations are in Appendix~\ref{app:beta_details}.

\section{Ada-Orth-Dion: Adaptive Rank as Intrinsic Dimensionality}
\label{sec:ada_dion}
Although Orth-Dion achieves convergence improvement with a minimal fix, the extra QR orthogonalization adds compute overhead per step.
To mitigate this and match Dion's wall-clock speed, we couple Orth-Dion with a per-layer adaptive rank mechanism to reduce the rank and the overhead, yielding \emph{Ada-Orth-Dion}. 

\textbf{From contraction to intrinsic dimensionality.}\;
The contraction condition $\tilde{\gamma}_{i,r}(1{+}\kappa_{G,i,r}) < 1$ governing Theorem~\ref{thm:orth_dion} decomposes into two competing terms per layer~$i$: $\tilde{\gamma}_{i,r}=\sigma_{r{+}1}(M_{i,t})/\sigma_r(M_{i,t})$ \emph{decreases} with $r$ (more signal captured), while $\kappa_{G,i,r}=\sigma_1(G_{i,t})/\sigma_r(G_{i,t})$ \emph{increases} with $r$ (the tail is harder to track). The critical rank $r_i^* = \max\{r: \tilde{\gamma}_{i,r}(1{+}\kappa_{G,i,r}) < 1\}$ marks the signal-to-noise transition and is the layer's \emph{intrinsic dimensionality}: above it, power iteration diverges; below it, the spectral update is provably contracting. 

\textbf{Lightweight estimator.}\;
We approximate $r_i^*$ from Dion's existing factors with no extra SVD: $\widehat{\erank}(M_t) = \exp(-\sum_i \hat{p}_i\log\hat{p}_i)$ where $\hat{p}_i = \hat{\sigma}_i/\sum_j\hat{\sigma}_j$ and $\hat{\sigma}_i = \|[W_t]_i\|_2$. The rank is updated via EMA smoothing ($\alpha$), a buffer multiplier ($\gamma{=}1.1$), clipping, and rounding to multiples of $8$ for GPU efficiency (full policy in Algorithm~\ref{alg:rank_policy}, Appendix~\ref{app:ada_dion_algo}).



\textbf{QR remains essential at adaptive rank.}\;
Reducing $r$ \emph{does not} substitute for fixing $\nu_t$.
Proposition~\ref{prop:colnorm_sqrt_r} bounds $\nu_t$ by the working rank, so a ColNorm-based adaptive scheme inherits the $\sqrt{r}$ penalty at whatever rank it selects.
Ada-Orth-Dion combines two complementary improvements, with QR removing the rank-dependent geometric inflation, as shown in Lemma~\ref{lem:partial_isometry}, and adaptive rank trimming the working rank to its intrinsic dimensionality.
Empirically (Table~\ref{tab:llm_main}), Ada-Orth-Dion settles at avg.\ rank $\sim\!357$ from a $r{=}384$ start, achieves the lowest validation loss of any method we evaluate that communicates strictly less than the parameter matrix ($3.282$), and reaches Dion's best loss $18.3\%$ faster than Dion at matched starting rank.

\section{Experiments}
\label{sec:experiments}

Our experiments answer: (1) Does replacing ColNorm with QR improve convergence at matched rank, as predicted by Lemma 4.1? (2) Does Ada-Orth-Dion recover Dion's per-step wall-clock on a larger model? (3) How does the error-feedback coefficient $\beta$ interact with the corrected geometry?

\subsection{LLM Pre-training: Llama~3 320M}
\label{sec:exp_llm}

\textbf{Setup.}\;
320M-param Llama~3~\citep{grattafiori2024llama} on C4~\citep{raffel2020t5}, 3.2B tokens (Chinchilla-optimal~\citep{hoffmann2022chinchilla}), seq len 2048, 8$\times$GH200 with FSDP.
All methods use cosine LR schedule with warmup, $\beta{=}1.0$ (unless stated), and identical hyperparameters except the normalization.
Baselines: AdamW ($\beta_1{=}0.95$, $\beta_2{=}0.95$), Muon (full-rank with Newton--Schulz), Dion ($r{\in}\{96,192,384\}$ with ColNorm), Orth-Dion (same ranks with QR), and Ada-Orth-Dion (adaptive rank, maximum rank $r_{max}$ at 384, rank initialized to $r_{max}$).

\begin{table}[t]
  \centering
  \caption{Llama~3 320M pre-training on C4 (3.2B tokens, 8$\times$GH200). Mean $\pm$ 2 std over 3 seeds. The steps reduction represents the percent reduction in steps to reach the Dion's minimum validation loss at equivalent rank.}
  \label{tab:llm_main}
  \resizebox{\linewidth}{!}{%
  \begin{tabular}{lccccc}
    \toprule
    \textbf{Method} & \textbf{Val.\ Ppl} & \textbf{Val.\ Loss} & $\bar{\nu}_t$ & \textbf{Avg.\ Rank} & \textbf{Steps Reduction} \\
    \midrule
    AdamW & 28.46{$\,\pm\,$1.59} & 3.348{$\,\pm\,$0.055} & --- & --- & --- \\
    Muon (Full, $r{=}768$) & 25.92{$\,\pm\,$0.25} & 3.255{$\,\pm\,$0.010} & 1.00 & full & --- \\
    \midrule
    Dion ($r{=}96$) & 29.58{$\,\pm\,$0.25} & 3.387{$\,\pm\,$0.009} & 1.20{$\,\pm\,$0.01} & 96 & --- \\
    Dion ($r{=}192$) & 28.37{$\,\pm\,$0.06} & 3.345{$\,\pm\,$0.002} & 1.28{$\,\pm\,$0.10} & 192 & --- \\
    Dion ($r{=}384$) & 27.39{$\,\pm\,$0.10} & 3.310{$\,\pm\,$0.004} & 1.34{$\,\pm\,$0.01} & 384 & --- \\
    \midrule
    Orth-Dion ($r{=}96$) & 29.37{$\,\pm\,$0.09} & 3.380{$\,\pm\,$0.003} & 1.00 & 96 & 13.1\%{$\,\pm\,$1.6\%} \\
    Orth-Dion ($r{=}192$) & 28.16{$\,\pm\,$0.13} & 3.338{$\,\pm\,$0.005} & 1.00 & 192 & 12.6\%{$\,\pm\,$1.9\%} \\
    Orth-Dion ($r{=}384$) & 27.20{$\,\pm\,$0.20} & 3.303{$\,\pm\,$0.007} & 1.00 & 384 & 12.3\%{$\,\pm\,$4.3\%} \\
    \midrule
    Ada-Orth-Dion ($r_{max}=384$) & 26.62{$\,\pm\,$0.10} & 3.282{$\,\pm\,$0.004} & 1.00 & 356.9{$\,\pm\,$12.2} & 18.3\%{$\,\pm\,$1.0\%} \\
    \bottomrule
  \end{tabular}%
  }
\end{table}

\textbf{Findings} (Table~\ref{tab:llm_main}).\;
\emph{(a) Matched-rank improvement.} Orth-Dion lowers val.\ loss over Dion at every rank and reaches Dion's best loss in $12$--$13\%$ fewer steps, which equates to $\sim10\%$ reduction in wall-clock time (Figure~\ref{fig:orth_speedup}). This confirms that replacing ColNorm with QR improves convergence.
\emph{(b) Direct theory validation.} Dion's measured $\bar{\nu}_t$ inflates monotonically with $r$ ($1.20{\to}1.28{\to}1.34$) while Orth-Dion has $\bar{\nu}_t = 1.00$ across all ranks, confirming Proposition~\ref{prop:colnorm_sqrt_r} and Lemma~\ref{lem:partial_isometry}.
\emph{(c) Adaptive rank.} Ada-Orth-Dion beats all low-rank baselines at an average rank of $356.9$ (93\% of Dion's $r=384$). The rank reduction can lead to wall-clock time improvement in large scale models, as discussed in Section~\ref{sec:exp_wallclock_17b}.

\subsection{\texorpdfstring{Measuring $\nu_t$ at LLM Scale}{Measuring nu sub t at LLM Scale}}

\label{sec:exp_sqrt_r}

The central theoretical prediction is that Dion's convergence degrades as $\sqrt{r}$ while Orth-Dion is rank-independent.
We validate this with two experiments.

\textbf{Mechanism check.}\;
Figure~\ref{fig:sqrt_r_scaling} (§\ref{sec:sqrt_r}) measures $\nu_t = \opnorm{\bar{V}_t}$ for every matrix layer and logged step during Llama~3 320M pretraining at $r\in\{96,192,384\}$. Both the time-course (Figure~\ref{fig:nu_bar}) and per-update distribution (Figure~\ref{fig:nu_hist}) confirm Proposition~\ref{prop:colnorm_sqrt_r}: Dion's $\nu_t$ inflates monotonically with $r$ throughout training, while Orth-Dion saturates the partial-isometry bound $\nu_t = 1$ across every (rank, layer, step) sample (Lemma~\ref{lem:partial_isometry}).

\textbf{Matched-rank, hyperparameter-controlled comparison.}\;
The one-step descent bound \eqref{eq:one_step} couples three rank-dependent terms: the captured-gradient mass $\kfnorm{G_t}{r}$ and the residual $\kfnorm{R_t}{r}$ both improve with $r$, while the smoothness penalty $L_r\nu_t^2\eta^2/2$ inflates with $\nu_t^2$. Their relative weight is regime-dependent, and the matched-$r$ Orth-Dion vs.\ Dion comparison, where the only difference is ColNorm vs.\ QR on Line~4 of Algorithm~\ref{alg:orth_dion}, is the clean isolation of the geometric correction.


On Llama 3 320M, the residual and descent gains from a larger $r$ outweigh the smoothness penalty, so Dion's val.\ loss falls as $r$ grows (Figure~\ref{fig:val_loss}). Dion's $\bar{\nu}_t$ does rise with rank, from $1.20$ at $r{=}96$ to $1.34$ at $r{=}384$, but the residual benefit at higher $r$ wins. With $r$ fixed, the only difference between Dion and Orth-Dion is ColNorm vs.\ QR, which translates directly to the dual-norm inflation. Orth-Dion beats Dion at every rank we tested (Table~\ref{tab:llm_main}). Learning rates, momentum, and schedules are identical across runs.


\begin{figure*}[t]
  \centering
  \centerline{\includegraphics[width=0.9\textwidth]{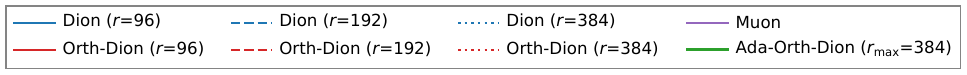}}
  \begin{subfigure}[t]{0.32\textwidth}
    \centering
    \includegraphics[width=\linewidth]{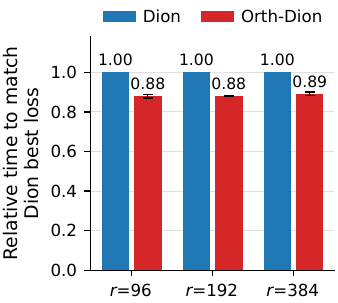}
    \caption{Wall-clock time to match Dion, normalized to mean Dion loss${=}1$. Orth-Dion: mean$\pm$2 std.}
    \label{fig:orth_speedup}
  \end{subfigure}\hfill
  \begin{subfigure}[t]{0.32\textwidth}
    \centering
    \includegraphics[width=\linewidth]{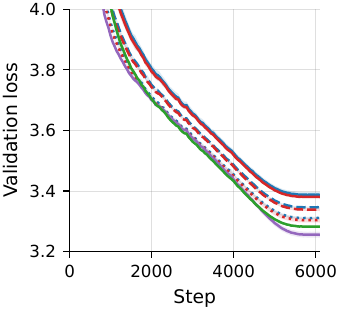}
    \caption{Validation loss on c4\_validation, logged every 50 steps.}
    \label{fig:val_loss}
  \end{subfigure}\hfill
  \begin{subfigure}[t]{0.32\textwidth}
    \centering
    \includegraphics[width=\linewidth]{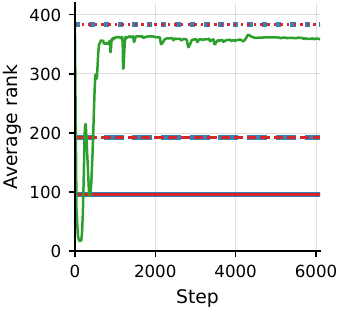}
    \caption{Mean per-block effective rank during training.}
    \label{fig:avg_rank}
  \end{subfigure}
    
    \caption{\footnotesize \textbf{Training dynamics for Llama 3 320M (6{,}100 steps, 8$\times$GH200 FSDP).} \textbf{(a)}~Orth-Dion reaches Dion's best loss at matched rank in $\sim0.89\times$ wall-clock time across $r\in\{96,192,384\}$. \textbf{(b)}~At every rank, Orth-Dion lands at lower C4 validation loss than Dion, consistent with Lemma~\ref{lem:partial_isometry}. \textbf{(c)}~Ada-Orth-Dion (init $r{=}384$) has $r\approx357$ on average; the fixed-rank curves are flat by construction.}
  \label{fig:main_4shards}
  \vspace{-2mm}
\end{figure*}

\subsection{Wall-Clock on Llama 3 17.1B}
\label{sec:exp_wallclock_17b}

Figure~\ref{fig:fig1_wallclock} measures per-step time on Llama 3 with 17.1B parameters on $8{\times}$A40 with FSDP full-shard. 
The largest per-block matrix is 16384$\times$43776, roughly 450$\times$ the size of the largest 320M matrix (768$\times$2048).
Combined with A40s' lower per-link bandwidth than NVLink-class accelerators, this pushes NCCL collectives into the bandwidth-bound regime, where communication volume dominates step time and the gap between low-rank and full-rank updates is most visible. 
Two findings: \emph{(i)} fixed-rank Orth-Dion is slightly slower than Dion at matched $r$ due to the QR step, and \emph{(ii)} Ada-Orth-Dion closes that gap and matches Dion's per-step time. Combined with the convergence advantage of Section~\ref{sec:exp_llm}, Ada-Orth-Dion gives the same per-step time as Dion at lower validation loss. Full architecture, hardware, and rank-policy details are in Appendix~\ref{app:llm_details:wide}.

\subsection{\texorpdfstring{Error Feedback: $\beta$-Sweep}{Error Feedback: beta-Sweep}}
\label{sec:exp_beta}

\begin{wraptable}{r}{0.46\textwidth}
\vspace{-\baselineskip}
\centering
\small
\caption{$\beta$-sweep on GPT-2 45M/WikiText-103. Smaller $\beta$ improves loss until $\beta{\approx}0.05$--$0.1$ while $\cR_t^*$ grows; disabling EF is worse. }
\label{tab:beta_sweep}
\begin{tabular}{lccc}
\toprule
\textbf{$\beta$} & \textbf{Val loss} & $\Delta$ \textbf{vs.}\ $\beta{=}1$ & $\cR_t^*$ \\
\midrule
0.05 & 2.982 & $-0.311$ & 4.6 \\
0.1  & \textbf{2.978} & $-0.315$ & 3.2 \\
0.3  & 3.056 & $-0.237$ & 1.7 \\
0.5  & 3.111 & $-0.183$ & 1.2 \\
1.0 (baseline) & 3.293 & --- & 0.3 \\
0 (no EF) & 3.554 & $+0.261$ & --- \\
\bottomrule
\end{tabular}
\vspace{-1em}
\end{wraptable}


Table~\ref{tab:beta_sweep} validates the $\beta{<}1$ regime on GPT-2 45M (wikitext-103, $r{=}64$, 30k steps), where $\phi_t \approx 0$ (Regime~2). Full $9$-point sweep in App.~\ref{app:beta_details}. Val loss decreases monotonically as $\beta$ is reduced from $1.0$ down to $\beta{\approx}0.1$, with the empirical optimum at $\beta{\approx}0.1$ (val.\ loss $2.978$ vs.\ $3.293$ for $\beta{=}1$), a $0.315$-nat improvement obtained at zero additional compute, in line with the theoretical prediction $\beta^*\!\ll\!1$. The R-norm ratio $\cR_t^*$ grows from $0.3$ at $\beta{=}1$ to $3.2$ at $\beta{=}0.1$, directly tracking the predicted implicit-momentum accumulation.
Disabling EF entirely ($\beta{=}0$) hurts by $0.261$ nats relative to $\beta{=}1$, confirming that EF's contribution is essential even when heavily damped.

\section{Analysis and Discussion}
\label{sec:analysis}

The $\sqrt{r}$ factor is rank-dependent, and rank is the parameter practitioners scale up. The realized $\nu_t$ in our Llama~3 runs sits well below the worst-case $\sqrt{r}$ envelope (Fig.~\ref{fig:sqrt_r_scaling}), so the absolute gap at any fixed rank is gentler than the worst case; nonetheless the inflation is monotone in $r$ and persists across layers and training steps. This gap is geometric rather than statistical. Lemma~\ref{lem:same_tracking} shows both algorithms track the same gradient subspace with identical error, so ColNorm does not lose information. It produces a direction that overshoots the dual-norm unit ball. The QR fix removes the inflaall inflation from the norm geometry of the update, independent of how tight the worst-case bound is.

\section{Related Work}
\label{sec:related}

\textbf{Spectral-norm optimizers.}\;
Shampoo~\citep{gupta2018shampoo} and SOAP~\citep{vyas2024soap} precondition with Kronecker factors
of second-moment statistics. 
Muon~\citep{jordan2024muon} instead takes the orthogonal polar factor of the gradient. That is the closed-form steepest-descent direction under the spectral norm in the modular-norm framework of~\citet{bernstein2024modular} computed via Newton–Schulz iteration.
All three require $O(mn)$ communication per matrix.


\textbf{Error feedback for biased compressors.}\;
Error feedback originates with 1-bit SGD~\citep{seide2014onebit} and was developed through
QSGD~\citep{alistarh2017qsgd}, sparsified SGD~\citep{stich2018sparsified}, and the EF21 framework~\citep{richtarik2021ef21}.
\citet{karimireddy2019error} showed EF restores convergence for sign-based methods.
All assume Euclidean smoothness, under which our setting's bounds are loose by $\min(m,n)/r$. 
Our $\phi_t$-framework adapts the EF analysis to operator-norm geometry and resolves the cross-step correlation between residual and signal that standard analyses implicitly absorb into a worst-case constant (Regimes 1 and 3).

\textbf{Low-rank and adaptive distributed training.}\;
PowerSGD~\citep{vogels2019powersgd} performs power iteration under Frobenius geometry, where column normalization is harmless.
The Ky Fan mismatch is specific to the operator-norm setting. 
Dion~\citep{ahn2025dion} inherits ColNorm from this lineage and (as we show) pays a $\sqrt{r}$ penalty for it.
Dion2~\citep{ahn2025dion2} selects row/column subsets without touching the normalization step. 
GaLore~\citep{zhao2024galore} and ReLoRA~\citep{lialin2023relora} operate in Frobenius geometry and so neither encounter nor benefit from the QR fix.
MuonBP~\citep{khaled2025muonbp} keeps full-rank updates but synchronizes only periodically, addressing communication along an orthogonal axis. 
For adaptive rank, AdaLoRA~\citep{zhang2023adalora} prunes singular values by importance during fine-tuning, and adaptive GaLore variants set rank from gradient-projection statistics.
Our contraction condition instead ties per-layer rank to the gradient spectrum's signal-to-noise transition in the from-scratch optimization regime, where the relevant geometry is operator-norm rather than Frobenius.

\section{Conclusion}
\label{sec:conclusion}

We identified a $\sqrt{r}$ geometric inefficiency in low-rank spectral optimization. Column normalization inflates the dual norm $\nu_t$ by up to $\sqrt{r}$, producing a convergence penalty that grows with rank. Replacing ColNorm with QR orthogonalization forces $\nu_t{=}1$ and recovers the exact polar rate targeted by Muon. Direct $\nu_t$ measurement on Llama~3 320M confirms the predicted rank-dependent inflation in Dion and that QR removes it identically.

Ada-Orth-Dion couples the geometric fix with per-layer adaptive rank that targets each layer's intrinsic dimensionality. 
The two corrections are complementary: QR removes the rank-dependent inflation, and adaptive rank reduces the working rank to offset QR's per-step compute overhead. 
It achieves the lowest validation loss ($3.282$) of any low-rank method evaluated, reaches within $0.027$ nats of full-rank Muon ($3.255$) at substantially lower per-step communication, and attains Dion's best loss in 18.3\% fewer steps. 
On the 17.1B model (Figure~\ref{fig:fig1_wallclock}, Appendix~\ref{app:llm_details:wide}), this convergence advantage holds at Dion-comparable per-step time.
Together, the fixed-rank Orth-Dion result establishes the geometric correction, and Ada-Orth-Dion turns it into a practical low-rank spectral optimizer with a strictly better convergence--communication trade-off than Dion.

\textbf{Limitations.}\;
Our analysis assumes deterministic gradients; extending to the stochastic setting requires bounding the interaction between mini-batch noise and subspace tracking.
The adaptive rank policy uses a heuristic estimator; tighter connections to the contraction condition would enable provable rank guarantees.


\section*{Acknowledgements}

This research was conducted using the Supermicro ARS-111GL-DNHR-LCC and FUJITSU Server PRIMERGY CX2550 M7 (Miyabi) at Joint Center for Advanced High Performance Computing (JCAHPC).
The authors also acknowledge the support of Compute Canada and Mila computing clusters for experimental resources.
We are grateful to the Masason Foundation for supporting this work through computational resources and for fostering a collaborative research environment.

\bibliographystyle{plainnat}
\bibliography{references.bib}

\newpage
\appendix
\section*{Appendix}
\section{Assumptions (Full Statements)}
\label{app:assumptions}

\begin{assumption}[Non-Euclidean smoothness]
\label{app:ass_smoothness}
There exists $L_r > 0$ such that for all $X, \Delta$ with $\mathrm{rank}(\Delta) \leq r$:
\[
f(X + \Delta) \leq f(X) + \ip{\nabla f(X)}{\Delta} + \frac{L_r}{2}\dualnorm{\Delta}^2.
\]
This uses the Ky Fan dual norm $\|\cdot\|_{(r)}$ rather than the Frobenius norm. For rank-$r$ updates $\Delta = U\bar{V}^\top$ with orthonormal $U$, $\bar{V}$: $\dualnorm{\Delta} = 1$, so the smoothness cost reduces to $L_r/2$ independent of $r$.

Assumption A is a one-sided function-value inequality. In Frobenius/Euclidean geometry, a matching descent lemma and a gradient-Lipschitz bound are equivalent; this equivalence \emph{fails} for the KF dual norm $\|\cdot\|_{(r)}$. Assumption A therefore does not by itself imply any bound of the form $\|\nabla f(Y)-\nabla f(X)\|_\bullet \le L_r\,\|Y-X\|_\bullet$. Steps that need such a bound (the gradient-subspace drift in Appendix~\ref{app:self_consistent_proof} and the resulting fixed-point argument) require Assumption A$'$ below.
\end{assumption}

\begin{assumption}[Gradient Lipschitz in Frobenius -- A$'$]
\label{app:ass_smoothness_F}
There exists $L_F > 0$ such that for all $X,Y \in \mathbb{R}^{m\times n}$:
\[
\fnorm{\nabla f(Y) - \nabla f(X)} \leq L_F \fnorm{Y-X}.
\]
This is the standard Euclidean gradient-Lipschitz condition. It is independent of Assumption A in general, and is the constant that drives the Wedin-type gradient-subspace drift bound used in Appendix~\ref{app:self_consistent_proof}. (Where the framework manuscript refers to a drift constant of $L_r\sqrt{r}$, the rigorous constant is $L_F\sqrt{r}$.)
\end{assumption}

\begin{assumption}[Gradient spectral gap -- B$'$]
\label{app:ass_gap}
There exist $\Delta_{\mathrm{gap}} > 0$ and $\sigma_{\min} > 0$ such that for all $t \geq 0$:
$\sigma_r(G_t) - \sigma_{r+1}(G_t) \geq \Delta_{\mathrm{gap}}$ and $\sigma_r(G_t) \geq \sigma_{\min}$.
This is a property of the \emph{gradient} (the optimization landscape), not the algorithm's internal state. It replaces the original Dion framework's Assumption B, which imposed a gap on the \emph{buffer} $M_t$---a quantity that depends on the algorithm's own error feedback and is therefore circular.

\emph{Local validity.} Assumption B$'$ is intended to hold over the regime in which the gradient still has its meaningful rank-$r$ structure (early/mid training). Near convergence, where $\kfnorm{G_t}{r}\to 0$, the lower bound $\sigma_r(G_t)\geq\sigma_{\min}$ may fail; the convergence results should be read as applying over any horizon on which B$'$ is in force.
\end{assumption}

\begin{assumption}[Small spectral tail -- B$''$]
\label{app:ass_tail}
There exists $\tau \geq 0$ such that for all $t\geq 0$: $\sigma_{r+1}(G_t) \leq \tau$, with $\tau = O(\eta)$ in the regime where the convergence theorem is asserted (cf.~Appendix~\ref{app:self_consistent_proof}).

\emph{Why this is needed.} The self-consistent residual recursion of Appendix~\ref{app:self_consistent_proof} yields the operator-norm bound $\opnorm{R_{t+1}} \leq \sigma_{r+1}(M_t) + \varepsilon_t \sigma_1(M_t)$. Assumption B$'$ controls the \emph{gap} but not the \emph{tail}; without B$''$, the floor $\sigma_{r+1}(M_t)$ is not $O(\eta)$, and the advertised $R_\infty = O(\eta)$ conclusion does not follow. With B$''$, the boxed rate of Theorem~\ref{thm:orth_dion} acquires an additive $O(\tau)$ term, which is absorbed when $\tau = O(\eta)$. When B$''$ fails (large tail), the rate becomes $\sqrt{2(f_0-f_\infty)L_r/T} + O(\tau) + O(1/T)$, i.e.\ a floor of order $\tau$.
\end{assumption}

\begin{assumption}[Gradient bounds -- C$'$]
\label{app:ass_bounds}
There exist $G_F, G_K, \kappa_G > 0$ such that for all $t \geq 0$:
$\fnorm{G_t} \leq G_F$, $\kfnorm{G_t}{r} \leq G_K$, $\kappa_r(G_t) \leq \kappa_G$.
\end{assumption}

\paragraph{Notational map for stability constants.}
Three closely related contraction symbols appear across the paper. To keep them straight:
\begin{itemize}
\item $\rho_t := \tilde{\gamma}_t(1+\kappa_{G,t})$ is the \emph{instantaneous} per-step contraction.
\item $\rho' < 1$ denotes a \emph{uniform per-step} bound $\rho_t \leq \rho'$ for all $t$ (used in Theorem~\ref{thm:orth_dion} for the cleanest statement).
\item $\bar{\rho} \in (0,1)$ together with $C_\rho \geq 1$ are the \emph{amortized / suffix-product} constants used when $\rho_t > 1$ is allowed transiently: $\prod_{i=s}^{t-1}\rho_i \le C_\rho\,\bar{\rho}^{t-s}$ for all $s<t$ (Section~\ref{sec:amortized}).
\end{itemize}
The uniform condition is the special case $\bar{\rho}=\rho'$, $C_\rho=1$. Theorem statements that say ``$\rho'<1$'' are tight conditions; the amortized version (Theorem~\ref{thm:amortized}) yields the same conclusion under the suffix-product condition above.

\section{Algorithm Details}
\label{app:algorithms}

\subsection{Full Algorithms}

\begin{algorithm}[H]
   \caption{Stripped Dion (Baseline)}
   \label{app:alg_stripped_dion}
\begin{algorithmic}[1]
   \STATE {\bfseries Input:} learning rate $\eta$, error-feedback coefficient $\beta$, rank $r$
   \FOR{$t=0,1,\dots,T-1$}
   \STATE $M_t \leftarrow G_t + R_t$ \hfill\textit{(buffer = gradient + residual)}
   \STATE $U_t \leftarrow \orth(M_t V_{t-1})$ \hfill\textit{(one step of subspace iteration)}
   \STATE $W_t \leftarrow M_t^\top U_t$, \quad $\bar{V}_t \leftarrow \ColNorm(W_t)$ \hfill\textit{(\textbf{column normalization})}
   \STATE $\hat{D}_t \leftarrow U_t \bar{V}_t^\top$ \hfill\textit{(rank-$r$ update direction)}
   \STATE $X_{t+1} \leftarrow X_t - \eta_t \hat{D}_t$
   \STATE $B_t \leftarrow U_t U_t^\top M_t$, \quad $R_{t+1} \leftarrow \beta\,(M_t - B_t)$ \hfill\textit{(error feedback)}
   \ENDFOR
\end{algorithmic}
\end{algorithm}

\begin{algorithm}[H]
   \caption{Orth-Dion (Proposed)}
   \label{app:alg_orth_dion}
\begin{algorithmic}[1]
   \STATE {\bfseries Input:} learning rate $\eta$, error-feedback coefficient $\beta$, rank $r$
   \FOR{$t=0,1,\dots,T-1$}
   \STATE $M_t \leftarrow G_t + R_t$ \hfill\textit{(buffer = gradient + residual)}
   \STATE $U_t \leftarrow \orth(M_t V_{t-1})$ \hfill\textit{(one step of subspace iteration)}
   \STATE $W_t \leftarrow M_t^\top U_t$, \quad $\bar{V}_t \leftarrow \orth(W_t)$ \hfill\textit{(\textbf{QR orthogonalization})}
   \STATE $\hat{D}_t \leftarrow U_t \bar{V}_t^\top$ \hfill\textit{(rank-$r$ update direction)}
   \STATE $X_{t+1} \leftarrow X_t - \eta_t \hat{D}_t$
   \STATE $B_t \leftarrow U_t U_t^\top M_t$, \quad $R_{t+1} \leftarrow \beta\,(M_t - B_t)$ \hfill\textit{(error feedback)}
   \ENDFOR
\end{algorithmic}
\end{algorithm}

\subsection{Ada-Orth-Dion Rank Selection}
\label{app:ada_dion_algo}

\begin{definition}[Estimated effective rank]
Given power-iteration factors $U_t \in \R^{m \times r}$ and $W_t = M_t^\top U_t \in \R^{n \times r}$, define $\hat{\sigma}_i = \|[W_t]_i\|_2$ and:
\[
\widehat{\erank}(M_t) = \exp\!\Big(-\sum_{i=1}^r \hat{p}_i \log \hat{p}_i\Big), \qquad \hat{p}_i = \frac{\hat{\sigma}_i}{\sum_j \hat{\sigma}_j}.
\]
\end{definition}

\begin{algorithm}[H]
\caption{Ada-Orth-Dion Rank Selection}
\label{alg:rank_policy}
\begin{algorithmic}[1]
\REQUIRE estimate $e_t$, current rank $r_t$, EMA $\bar{e}_{t-1}$, bounds $r_{\min}, r_{\max}$, smoothing $\alpha$, buffer $\gamma$
\STATE $\bar{e}_t \gets \alpha\, e_t + (1-\alpha)\,\bar{e}_{t-1}$
\STATE $r_{\mathrm{target}} \gets \lceil \gamma\,\bar{e}_t \rceil$
\STATE $r_{t+1} \leftarrow \clip(r_{\mathrm{target}},\; r_{\min},\; r_{\max})$, rounded to multiple of 8
\IF{$r_{t+1} \neq r_t$}
  \STATE Resize $U, V$ to rank $r_{t+1}$ (truncate or pad)
\ENDIF
\end{algorithmic}
\end{algorithm}

\textbf{Connection to contraction condition.}\;
The estimated effective rank is a spectral entropy measure. When the gradient has a sharp spectral gap at position $k$, the effective rank concentrates near $k$ regardless of the nominal $r$. By tracking $\widehat{\erank}$ and adjusting $r_t$, Ada-Dion approximates $r_i^* = \max\{r : \tilde{\gamma}_{i,r}(1+\kappa_{G,i,r}) < 1\}$ without explicit SVD.

\section{Norms and Key Quantities}
\label{app:norms}

For $A \in \R^{m \times n}$ with singular values $\sigma_1(A) \geq \sigma_2(A) \geq \ldots$:

\begin{center}
\small
\begin{tabular}{lll}
\toprule
\textbf{Notation} & \textbf{Name} & \textbf{Definition} \\
\midrule
$\fnorm{A}$ & Frobenius & $\sqrt{\sum_{i,j}A_{ij}^2} = \sqrt{\sum_i \sigma_i^2}$ \\
$\opnorm{A}$ & Operator (spectral) & $\sigma_1(A)$ \\
$\nucnorm{A}$ & Nuclear & $\sum_i \sigma_i(A)$ \\
$\kfnorm{A}{r}$ & Ky Fan $r$-norm & $\sum_{i=1}^r \sigma_i(A)$ \\
$\dualnorm{A}$ & Dual of KF $r$-norm & $\max\{\sigma_1(A),\; \nucnorm{A}/r\}$\,\footnote{The earlier formula $\max\{\sigma_1(A),\,\fnorm{A}/\sqrt{r}\}$ underestimates the dual norm in general. Counterexample: $r=2$, $A=I_3$ gives $\max\{1,\sqrt{3/2}\}=\sqrt{3/2}$ via the old formula, while the correct dual norm is $\max\{1,3/2\}=3/2$, attained by $B=\tfrac12 I_3$ (which has $\kfnorm{B}{2}=1$ and $\ip{A}{B}=3/2$).} \\
$\kappa_r(A)$ & Condition at rank $r$ & $\sigma_1(A)/\sigma_r(A)$ \\
\bottomrule
\end{tabular}
\end{center}

\textbf{Key H\"older identity.}\;
$\ip{A}{B} \leq \kfnorm{A}{r} \cdot \dualnorm{B}$, with equality \emph{attained} at $B=P_r(A)$ (the rank-$r$ polar factor): then $\dualnorm{P_r(A)}=1$ and $\ip{A}{P_r(A)}=\kfnorm{A}{r}$. Equality is not in general unique --- e.g.\ for $r{=}1$ and $A=\diag(1,1)$, both $B_1=\diag(1,0)$ and $B_2=\diag(0,1)$ saturate the bound.

\textbf{Key quantities per iteration.}\;
$\varepsilon_t = \opnorm{U_tU_t^\top - U_{r,t}U_{r,t}^\top}$ (projector error),
$\delta_t = \kfnorm{M_t}{r} - \ip{M_t}{\hat{D}_t}$ (oracle defect),
$\nu_t = \dualnorm{\hat{D}_t}$ (dual norm factor),
$\tilde{\gamma}_t = \sigma_{r+1}(M_t)/\sigma_r(M_t)$ (buffer spectral ratio).

\section{Proof of One-Step Descent (Equation~\ref{eq:one_step})}
\label{app:one_step}

\begin{proof}
\textbf{Step 1 (Smoothness):}
$f(X_{t+1}) \leq f(X_t) + \ip{G_t}{-\eta\hat{D}_t} + \frac{L_r}{2}\dualnorm{\eta\hat{D}_t}^2 = f(X_t) - \eta\ip{G_t}{\hat{D}_t} + \frac{L_r\nu_t^2}{2}\eta^2$.

\textbf{Step 2 (Gradient-buffer decomposition):}
$\ip{G_t}{\hat{D}_t} = \ip{M_t}{\hat{D}_t} - \ip{R_t}{\hat{D}_t}$.

\textbf{Step 3 (LMO defect):}
$\ip{M_t}{\hat{D}_t} = \kfnorm{M_t}{r} - \delta_t$.

\textbf{Step 4 (Triangle inequality):}
$\kfnorm{M_t}{r} \geq \kfnorm{G_t}{r} - \kfnorm{R_t}{r}$.

\textbf{Step 5 (H\"older for $R_t$ term):}
$|\ip{R_t}{\hat{D}_t}| \leq \kfnorm{R_t}{r} \cdot \nu_t$.

\textbf{Assembly:}
$-\ip{G_t}{\hat{D}_t} \leq -\kfnorm{G_t}{r} + \delta_t + (1+\nu_t)\kfnorm{R_t}{r}$.
Substituting into Step~1 completes the proof.
\end{proof}

\section{Orth-Dion Proofs}
\label{app:orth_dion_proofs}

\begin{proof}[Proof of Lemma~\ref{lem:partial_isometry}]
$\hat{D}^\top\hat{D} = \bar{V}U^\top U\bar{V}^\top = \bar{V}\bar{V}^\top$ is an orthogonal projector with eigenvalues $\{0,1\}$.
Therefore $\hat{D}$ has $r$ singular values equal to $1$ and the rest $0$. Hence $\sigma_1(\hat D)=1$ and $\nucnorm{\hat D}=r$, so $\dualnorm{\hat D} = \max\{1,\;\nucnorm{\hat D}/r\} = \max\{1,1\} = 1$.
\end{proof}

\begin{proof}[Proof of Lemma~\ref{lem:nonneg_defect}]
By H\"older: $\ip{M_t}{\hat{D}_t} \leq \kfnorm{M_t}{r}\cdot\dualnorm{\hat{D}_t} = \kfnorm{M_t}{r}\cdot 1$.
So $\delta_t = \kfnorm{M_t}{r} - \ip{M_t}{\hat{D}_t} \geq 0$.

\textbf{Significance}: For ColNorm, $\nu_t > 1$ means $\hat{D}_t$ can violate the H\"older bound, making $\delta_t$ potentially negative and complicating the analysis.
\end{proof}

\begin{proof}[Proof of Lemma~\ref{lem:same_tracking}]
The tracking recursion depends on three quantities:
(1) Column space of $V_{t-1}$: $\col(\orth(W)) = \col(W) = \col(\ColNorm(W))$ for full-rank $W$.
(2) Spectral gap of $M_t$: unchanged between algorithms.
(3) Buffer drift $\opnorm{M_t - M_{t-1}}$: depends on $\eta, \beta$, unchanged.
Therefore $\varepsilon_t$ satisfies the same recursion and bounds.
\end{proof}

\textbf{Defect bound for Orth-Dion.}\;
Under the projector-level bound with Lemma~\ref{lem:same_tracking}:
$0 \leq \delta_t \leq C'_\delta\fnorm{M_t}\cdot\varepsilon_t$ where $C'_\delta = (1+\kappa_G)\sqrt{2r}$.
The constant is the same order as for Stripped Dion; the improvement comes entirely from $\nu_t$.

\section{Self-Consistent Drift Elimination (Full Details)}
\label{app:self_consistent_proof}

\subsection{Why the Original Assumption D Was Circular}

Assumption D required $\opnorm{M_t - M_{t-1}} \leq \Delta_{\mathrm{drift}}$.
But $M_t - M_{t-1} = (G_t - G_{t-1}) + (R_t - R_{t-1})$.
Under Assumption A$'$ (Frobenius gradient-Lipschitz, Appendix~\ref{app:ass_smoothness_F}), the gradient difference is small: $\fnorm{G_{t+1}-G_t} \le L_F \fnorm{X_{t+1}-X_t} = L_F\eta\fnorm{\hat D_t} \le L_F\eta\sqrt{r}$.
The residual difference is \emph{not} small: $R_{t+1} - R_t = \beta\,(I - U_tU_t^\top)M_t - R_t$, so $\opnorm{R_{t+1} - R_t} = O(\|G_t\|)$.
Assumption D is inconsistent with the algorithm's actual dynamics.

\subsection{Gradient Subspace Stability}

\begin{proposition}[Gradient subspace drift]
\label{app:prop_grad_drift}
Under Assumptions A$'$ and B$'$: $\sin\Theta_{\max}(P_G^{t+1}, P_G^t) \leq L_F\,\eta\,\sqrt{r}/\Delta_{\mathrm{gap}}$.
\end{proposition}

\begin{proof}
By A$'$: $\fnorm{G_{t+1} - G_t} \leq L_F\,\fnorm{X_{t+1}-X_t} = L_F\eta\fnorm{\hat D_t} \leq L_F\eta\sqrt{r}$, where the last step uses $\fnorm{\hat D_t}\leq\sqrt{r}$ (every $\hat D_t$ has at most $r$ singular values, each at most $1$). Wedin's $\sin\Theta$ theorem with eigengap $\Delta_{\mathrm{gap}}$ then yields the stated bound.

\emph{Remark on the constant.} An earlier version of this proof attempted to derive a Frobenius gradient-Lipschitz bound from Assumption A. That step is not valid: A is a one-sided function-value smoothness inequality in the KF dual norm and does not imply a Frobenius gradient-Lipschitz estimate. We therefore introduce A$'$ as a separate assumption and state the drift bound with constant $L_F\sqrt{r}$, not $L_r\sqrt{r}$. This substitution propagates through all downstream uses of the drift bound (in particular, the fixed-point map below).
\end{proof}

\subsection{Buffer Spectral Gap Transfer}

\begin{proposition}
Under Assumptions B', C', if $\opnorm{R_t} \leq \Delta_{\mathrm{gap}}/2$:
(1) $\sigma_r(M_t) \geq \sigma_{\min} - \Delta_{\mathrm{gap}}/2 > 0$;
(2) $\tilde{\gamma}_{\mathrm{eff}} \leq (\sigma_{r+1}(G_t) + \opnorm{R_t})/(\sigma_r(G_t) - \opnorm{R_t})$;
(3) Buffer gap: $\sigma_r(M_t) - \sigma_{r+1}(M_t) \geq \Delta_{\mathrm{gap}} - 2\opnorm{R_t}$.
All follow from Weyl's inequality.
\end{proposition}

\subsection{Self-Consistent Residual Bound}

With $\beta = 1$ and projector error $\varepsilon_t$, the residual update yields the operator-norm bound
\[
\opnorm{R_{t+1}} \;\le\; \sigma_{r+1}(M_t) + \varepsilon_t\cdot\sigma_1(M_t).
\]
This decomposition has two components: the spectral \emph{tail} $\sigma_{r+1}(M_t)$ (a property of the buffer's low-rank approximation error, controlled by Assumption B$''$ together with Weyl-type buffer-vs-gradient comparison), and the tracking-leakage term $\varepsilon_t\sigma_1(M_t)$ (controlled by the contraction $\rho'<1$).

\emph{Status (added 2026-04-15).} Assumption B$'$ controls only the gradient gap, not the tail; even with perfect tracking ($\varepsilon_t=0$) the floor $\sigma_{r+1}(M_t)$ remains, and Weyl gives only $\sigma_{r+1}(M_t) \le \sigma_{r+1}(G_t) + \opnorm{R_t}$. The advertised conclusion $R_\infty = O(\eta)$ therefore requires the additional small-tail Assumption B$''$ (Appendix~\ref{app:ass_tail}); with $\sigma_{r+1}(G_t)\leq \tau$, the recursion has a fixed point $R_\infty = O(\tau+\eta)$, which collapses to $O(\eta)$ when $\tau=O(\eta)$. When B$''$ fails (e.g.\ on heavy-tailed gradients), the conclusion weakens to a tail floor of order $\tau$ in Theorem~\ref{thm:orth_dion}.

\subsection{The Fixed Point}

\begin{proposition}[Self-consistent fixed point]
\label{prop:self_consistent}
Assume A, A$'$, B$'$, B$''$ (with tail $\tau$), C$'$, and let $\rho' = \tilde{\gamma}_{\mathrm{eff}}(1+\kappa_G) < 1$. For $\eta = O(\Delta_{\mathrm{gap}}/(L_F\sqrt{r}\kappa_G))$, the coupled system
\begin{align*}
\varepsilon'_\infty &= \frac{\tilde{\gamma}_{\mathrm{eff}} \cdot D_{\mathrm{eff}}}{1 - \rho'}, \quad
R_\infty \leq \underbrace{\sigma_{r+1}(M)}_{\le\,\tau\,+\,\opnorm{R_\infty}\,(\text{B$''$+Weyl})} + \varepsilon'_\infty\sigma_1, \\
D_{\mathrm{eff}} &= \frac{L_F\eta\sqrt{r} + R_\infty}{\Delta_{\mathrm{gap}} - 2R_\infty}, \quad
\tilde{\gamma}_{\mathrm{eff}} = \frac{\sigma_{r+1}(G) + R_\infty}{\sigma_r(G) - R_\infty}
\end{align*}
has a unique small fixed point with $R_\infty = O(\tau + \eta)$ and $\varepsilon'_\infty = O(\eta)$. In particular, when $\tau = O(\eta)$, both quantities are $O(\eta)$.
\end{proposition}

\begin{proof}
Banach iteration from $(0,0)$. First iterate: $D_{\mathrm{eff}}^{(0)} = L_F\eta\sqrt{r}/\Delta_{\mathrm{gap}} = O(\eta)$ (uses A$'$), $\tilde{\gamma}_{\mathrm{eff}}^{(0)} = \tilde{\gamma}_G$, $\varepsilon'^{(0)} = O(\eta)$, $R^{(1)} \leq \tau + O(\eta)$ (using B$''$ to bound the tail floor). Each subsequent iterate perturbs by $O(\eta^2)$. The map is Lipschitz with constant $O(\eta)$ near the fixed point, hence contractive.

\emph{What this proves precisely.} Under A, A$'$, B$'$, B$''$, C$'$, the coupled system has a fixed point with $R_\infty,\varepsilon'_\infty$ as stated, on any finite horizon over which the assumptions hold. Assumption D is \emph{derived} in this regime, not posited. Outside this regime --- in particular when B$''$ fails --- the drift is not eliminated; it is re-expressed in terms of the tail $\tau$, which can be large.
\end{proof}

\section{Amortized Contraction (Full Proof)}
\label{app:amortized_proof}

\begin{theorem}[Amortized contraction]
\label{thm:amortized}
Let $\rho_t = \tilde{\gamma}_t(1+\kappa_{G,t})$. Suppose there exist constants $\bar{\rho}\in(0,1)$ and $C_\rho \geq 1$ such that the \emph{suffix-product condition} holds: for every $0\leq s < t$,
\[
\prod_{i=s}^{t-1} \rho_i \;\leq\; C_\rho\,\bar{\rho}^{\,t-s}.
\]
Then the conclusion of Theorem~\ref{thm:orth_dion} holds with $\rho'$ replaced by $\bar{\rho}$ and an additional multiplicative constant $C_\rho$ in the tracking-error and residual-coupling contributions.
\end{theorem}

\begin{proof}
The tracking recursion with variable $\rho_t$ is $\varepsilon_t \le \rho_{t-1}\varepsilon_{t-1} + c_{t-1}$. Unrolling,
\[
\varepsilon_T \;\le\; \Big(\prod_{i=0}^{T-1}\rho_i\Big)\varepsilon_0 \;+\; \sum_{s=0}^{T-1} c_s\,\Big(\prod_{i=s+1}^{T-1}\rho_i\Big).
\]
Apply the suffix-product condition twice. \emph{Initial term:} $\prod_{i=0}^{T-1}\rho_i \le C_\rho\,\bar{\rho}^{T} \to 0$. \emph{Sum term:} each partial product $\prod_{i=s+1}^{T-1}\rho_i \le C_\rho\,\bar{\rho}^{T-1-s}$, so $\sum_{s} c_s\prod_{i=s+1}^{T-1}\rho_i \le C_\rho\,\frac{\max_s c_s}{1-\bar{\rho}}$. Both terms are uniformly controlled, and the rest of the proof of Theorem~\ref{thm:orth_dion} (Appendix~\ref{app:main_proof}) goes through with $1/(1-\rho')$ replaced by $C_\rho/(1-\bar{\rho})$.

\emph{Why the suffix-product condition (not the time-averaged $\bar{\rho}=\exp((1/T)\sum_t\log\rho_t)$) is required.}
A condition on the \emph{full-horizon} geometric mean does not control \emph{suffix} products: a sequence with $\rho_t<1$ for $t<T/2$ and $\rho_t>1$ for $t\ge T/2$ can have full-horizon mean $<1$ while every suffix product over $[T/2,T)$ grows. The earlier statement ``$\bar{\rho}=\exp((1/T)\sum_t\log\rho_t)<1$ implies each partial product is bounded by a power of $\bar{\rho}$ by AM-GM'' is therefore not valid as written. The strong suffix-product condition above is exactly the body's Section~\ref{sec:amortized} statement and is the natural Lyapunov-style stability condition for subspace tracking.
\end{proof}

\section{Complete Proof of Theorem~\ref{thm:orth_dion}}
\label{app:main_proof}

\begin{proof}
\textbf{Step 1 (Telescoping).}\;
Sum the Orth-Dion descent bound over $t = 0, \ldots, T{-}1$:
\[
\sum_t \eta\kfnorm{G_t}{r} \leq (f_0 - f_\infty) + \sum_t \eta(\delta_t + 2\kfnorm{R_t}{r}) + \tfrac{L_r}{2}T\eta^2.
\]

\textbf{Step 2 (Error terms, $\beta{=}1$).}\;
By the self-consistent analysis: $\varepsilon'_\infty = O(\eta)$.
Defect: $\delta_t \leq C'_\delta G_F \varepsilon_t$.
Residual: $\kfnorm{R_{t+1}}{r} \leq C_R G_F \varepsilon_t$.
Error sum splits: $\sum_t\varepsilon_t \leq \varepsilon_0/(1{-}\rho') + T\varepsilon'_\infty = O(1) + O(T\eta)$.
Therefore $\sum_t\eta(\delta_t + 2\kfnorm{R_t}{r}) = O(T\eta^2)$.

\textbf{Step 3 (Assembly).}\;
$\eta T\min_t\kfnorm{G_t}{r} \leq (f_0-f_\infty) + (\frac{L_r}{2}+C)T\eta^2$.
Dividing: $\min_t\kfnorm{G_t}{r} \leq \frac{f_0-f_\infty}{\eta T} + \frac{L_r\eta}{2} + O(1/T)$.

\textbf{Step 4 (Optimize $\eta = c/\sqrt{T}$).}\;
$\min_c[\frac{f_0-f_\infty}{c} + \frac{L_r c}{2}]$ at $c^* = \sqrt{2(f_0-f_\infty)/L_r}$.
Substituting: $\min_t\kfnorm{G_t}{r} \leq \sqrt{2(f_0-f_\infty)L_r/T} + O(1/T)$.
\end{proof}

\section{Convergence Rate Comparison}
\label{app:rate_comparison}

The exact-Muon column in Table~\ref{tab:rate_comparison} should be read as a geometry-level comparison in the rank-$r$ Ky Fan setting. It is complementary to recent full-rank Muon analyses, which prove convergence for practical Muon variants with Nesterov momentum and weight decay and derive critical batch sizes~\citep{sato2025muoncbs}. Our focus is instead the low-rank spectral update geometry and the rank-dependent dual-norm inflation caused by column normalization.

\begin{table}[h]
\centering
\caption{Complete rate comparison across algorithm variants.}
\label{tab:rate_comparison}
\small
\begin{tabular}{lccc}
\toprule
\textbf{Property} & \textbf{Stripped Dion} & \textbf{Orth-Dion} & \textbf{Exact Muon} \\
\midrule
$\nu_t$ & $\leq \sqrt{r}$ & $= 1$ & $= 1$ \\
$\delta_t$ & Signed & $\geq 0$ & $= 0$ \\
Smoothness/step & $\frac{L_r r}{2}\eta^2$ & $\frac{L_r}{2}\eta^2$ & $\frac{L_r}{2}\eta^2$ \\
Residual coupling & $(1{+}\sqrt{r})|R|$ & $2|R|$ & $0$ \\
\textbf{Rate} & $O(\sqrt{L_r r/T})$ & $O(\sqrt{L_r/T})$ & $O(\sqrt{L_r/T})$ \\
Iterations to $\varepsilon$ & $O(L_r r/\varepsilon^2)$ & $O(L_r/\varepsilon^2)$ & $O(L_r/\varepsilon^2)$ \\
Per-step cost & $O(mnr)$ & $O(mnr){+}O(nr^2)$ & $O(mn\min\{m,n\})$ \\
\bottomrule
\end{tabular}
\end{table}

\section{Contraction Condition Comparison}
\label{app:contraction}

The convergence proof requires the per-step buffer-tracking contraction rate $\tilde{\gamma}$ to lie strictly below an admissibility ceiling that depends on how tightly the residual amplification can be controlled. The original Dion ceiling scales as $1/(1+2\sqrt{2r}\kappa)$, so it shrinks as the rank grows. Our analysis replaces this with the rank-independent ceiling $1/(1+\kappa)$, which depends only on the conditioning of the gradient covariance. Table~\ref{tab:contraction_compare} reports both ceilings at the representative regime $r{=}64$, $\kappa{=}10$ used throughout our LLM experiments: the original bound forces $\tilde{\gamma}<0.0044$ while ours permits $\tilde{\gamma}<0.091$, a $22.6\times$ relaxation.

\begin{table}[h]
\centering
\caption{Contraction condition at $r{=}64$, $\kappa{=}10$. Our Ky Fan dual-norm analysis (Section~\ref{sec:theory}) relaxes the original Dion bound by $22.6\times$.}
\label{tab:contraction_compare}
\small
\begin{tabular}{lccc}
\toprule
\textbf{Framework} & \textbf{Condition} & $r{=}64$ & \textbf{Relaxation} \\
\midrule
Original Dion & $\tilde{\gamma} < 1/(1+2\sqrt{2r}\kappa)$ & $\tilde{\gamma} < 0.0044$ & --- \\
Ours & $\tilde{\gamma} < 1/(1+\kappa)$ & $\tilde{\gamma} < 0.091$ & $22.6\times$ \\
\bottomrule
\end{tabular}
\end{table}

\section{\texorpdfstring{$\beta < 1$ Regimes: Additional Details}{beta less than 1 Regimes: Additional Details}}
\label{app:beta_details}

This appendix supplements Section~\ref{sec:beta_regimes} with the full three-regime taxonomy of error-feedback dynamics, the complete $\beta$-sweep used to validate Regime~2 on GPT-2~45M, and the implicit-momentum derivation behind the prediction $\beta^*\approx 1-2\hat{\varepsilon}$. With $\nu_t$ fixed by Orth-Dion, the residual recursion $R_{t+1}=\beta(I-U_tU_t^\top)(G_t+R_t)$ is driven entirely by the cross-step correlation of the out-of-subspace gradient, summarised through the persistence $\phi_t:=\cos\angle((I{-}P_t)G_t,(I{-}P_{t-1})G_{t-1})$. The sign of $\phi_t$ determines which of three qualitatively distinct dynamical regimes the optimizer is operating in, and therefore which value of $\beta$ is optimal.

Table~\ref{tab:three_regimes} summarises the three regimes. In Regime~1 ($\phi_t>0$), gradients are coherent across steps, residuals accumulate constructively, and the tracked subspace eventually rotates to absorb them; full retention ($\beta=1$) is optimal because it preserves the coherent signal. In Regime~2 ($\phi_t\approx 0$, the typical setting for stochastic LLM pretraining), residual increments are nearly independent, so $R_T$ executes a $\sqrt{T}$ random walk that, when projected back onto the tracked subspace, behaves like Polyak momentum with effective coefficient $\hat{\varepsilon}^2$; reducing $\beta$ amplifies this implicit-momentum signal at zero additional compute. In Regime~3 ($\phi_t<0$), residuals anti-correlate across steps and the EF accumulator chases a moving target, so vanishing $\beta$ is required to suppress harmful interference.

\begin{table}[h]
\centering
\caption{Three regimes of error feedback dynamics, governed by the gradient persistence $\phi_t$.}
\label{tab:three_regimes}
\small
\begin{tabular}{lcccc}
\toprule
\textbf{Regime} & $\phi_t$ & \textbf{R-norm} & \textbf{EF effect} & \textbf{Optimal $\beta$} \\
\midrule
1.~Coherent & $> 0$ & accumulates & strongly beneficial & $\beta = 1$ \\
2.~Stochastic & $\approx 0$ & $\sqrt{T}$ random walk & implicit momentum & $\beta \approx 0.3$ \\
3.~Anti-correlated & $< 0$ & diverges & harmful & $\beta \to 0$ \\
\bottomrule
\end{tabular}
\end{table}

Table~\ref{tab:beta_sweep_full} reports the complete nine-point $\beta$-sweep on GPT-2~45M (WikiText-103, $r{=}64$, 30k steps), of which the body Table~\ref{tab:beta_sweep} is a six-point subset. All runs share identical optimizer hyperparameters apart from $\beta$, and the empirical persistence $\phi_t\approx-0.01$ places the configuration squarely in Regime~2. Validation loss decreases monotonically as $\beta$ shrinks from $1.0$, with the empirical optimum at $\beta{\approx}0.1$ (val.\ loss $2.978$ vs.\ $3.293$ at $\beta{=}1$); the R-norm ratio $\cR_t^*$ grows monotonically from $0.3$ at $\beta{=}1$ to $4.6$ at $\beta{=}0.05$, directly tracking the predicted implicit-momentum accumulation. Disabling EF entirely ($\beta{=}0$) is far worse than any $\beta{>}0$ in the sweep, indicating that the gain from $\beta{<}1$ comes from tuning the implicit-momentum coefficient rather than from removing the residual altogether.

\begin{table}[h]
\centering
\caption{Full $\beta$-sweep on GPT-2 45M/WikiText-103 ($r{=}64$, 30k steps; Regime~2, $\phi_t{\approx}0$). Validation loss is monotonically non-increasing in $\beta$ from $\beta{=}1$ down to $\beta{\approx}0.05$--$0.1$; below this the implicit-momentum mechanism saturates (Section~\ref{sec:beta_regimes}). Disabling EF entirely ($\beta{=}0$) is substantially worse than any $\beta{>}0$, confirming EF is essential even when heavily damped. Body Table~\ref{tab:beta_sweep} reports the subset \{0, 0.05, 0.1, 0.3, 0.5, 1.0\}.}
\label{tab:beta_sweep_full}
\small
\begin{tabular}{lccc}
\toprule
\textbf{$\beta$} & \textbf{Val loss} & $\Delta$ \textbf{vs.}\ $\beta{=}1$ & $\cR_t^*$ \\
\midrule
0.05 & 2.982 & $-0.311$ & 4.6 \\
0.1  & \textbf{2.978} & $-0.315$ & 3.2 \\
0.2  & 3.019 & $-0.274$ & 2.2 \\
0.3  & 3.056 & $-0.237$ & 1.7 \\
0.47 & 3.076 & $-0.217$ & 1.3 \\
0.5  & 3.111 & $-0.183$ & 1.2 \\
0.7  & 3.177 & $-0.116$ & 0.9 \\
1.0 (baseline) & 3.293 & --- & 0.3 \\
0 (no EF) & 3.554 & $+0.261$ & --- \\
\bottomrule
\end{tabular}
\end{table}

\subsection{Key Quantities}

\textbf{Tracking error:} $\hat{\varepsilon}_t := \|(I - U_tU_t^\top)M_t\|_F / \fnorm{M_t}$.
For ColNorm: $\hat{\varepsilon}_t \approx 0.35$ in LLM training.
For Orth-Dion: $\hat{\varepsilon}_t \approx 0.007$.

\textbf{Gradient subspace persistence:}
$\phi_t := \cos\angle((I{-}P_t)G_t,\; (I{-}P_{t-1})G_{t-1})$ (distinct from the per-step contraction $\rho_t$ of Section~\ref{sec:amortized}).
$\phi_t > 0$: directions coherent.
$\phi_t \approx 0$: approximately independent (LLM typical: $\phi_t \approx -0.01$).
$\phi_t < 0$: anti-aligned.

\textbf{R-norm ratio:} $\cR_t := \fnorm{R_t}/\fnorm{G_t}$.
Empirically $\cR_t \approx 0$ for $\beta{=}1$; grows monotonically as $\beta$ decreases.
Optimal sweet spot: $\cR_t^* \in [1.5, 2.0]$.

\subsection{Implicit Momentum Mechanism (Regime 2)}

When $\phi_t \approx 0$, the residual $R_T$ is a sum of approximately independent terms:
$R_T \approx \sum_{t=0}^{T-1}(I-P_t)G_t$.
By random-walk scaling: $\fnorm{R_T} \sim \sqrt{T}\cdot\hat{\varepsilon}\cdot\fnorm{G}$.

The projected residual $P_T R_T$ acts as a momentum term:
$\E[P_T R_T] \approx \hat{\varepsilon}^2 \cdot \sum_{t<T}\gamma^{T-t}\nabla f(w_t)$
for effective decay $\gamma \in (0,1)$ depending on subspace rotation rate.
This is equivalent to heavy-ball momentum with coefficient $\hat{\varepsilon}^2$.

For ColNorm ($\hat{\varepsilon} \approx 0.35$): implicit momentum $\approx 0.12$---substantial.
For Orth-Dion ($\hat{\varepsilon} \approx 0.007$): implicit momentum $\approx 5\times10^{-5}$---negligible.
\emph{This is why ColNorm+EF empirically outperforms Orth-Dion+EF even when $\phi_t \approx 0$.}

\subsection{\texorpdfstring{Experimental Evidence: EF $\neq$ Explicit Momentum}{Experimental Evidence: EF not equal to Explicit Momentum}}

PolyakDion (explicit Polyak momentum $\mu{=}0.95$, $R{\equiv}0$ by design) achieves val loss 3.162.
Dion with $\beta{=}0.3$ achieves 3.056---a gap of 0.106 nats.
The mechanism: EF accumulation reshapes $M_t = G_t + R_t$ \emph{before} the rank-$r$ compression step, directly influencing which subspace is tracked.
Explicit momentum is applied before compression and does not interact with subspace selection in the same way.



\section{LLM Pre-training Details}
\label{app:llm_details}

Our code will be available at 
\url{https://anonymous.4open.science/r/orth_dion-4DD4/README.md}.

\subsection{320M Pre-training}
\label{app:llm_details:320m}

\begin{table}[h]
    \centering
    \caption{Llama~3 320M architecture.}
    \begin{tabular}{lc}
        \toprule
        \textbf{Hyperparameter} & \textbf{Value} \\
        \midrule
        Hidden dimension & 768 \\
        Intermediate dimension (FFN) & 2048 \\
        Number of layers & 18 \\
        Number of attention heads & 12 \\
        Number of KV heads & 12 \\
        Sequence length & 2048 \\
        Vocabulary size & 128{,}256 \\
        Total parameters & $\sim$320M \\
        \bottomrule
    \end{tabular}
    \label{tab:llm_details:model}
\end{table}
\begin{table}[h]
    \centering
    \caption{Training hyperparameters (320M).}
    \begin{tabular}{lc}
        \toprule
        \textbf{Hyperparameter} & \textbf{Value} \\
        \midrule
        Token budget & 3.2B \\
        Warmup steps & 2000 \\
        LR schedule & cosine decay \\
        Gradient clipping & 1.0 \\
        Batch size & 256 \\
        GPUs & 8 $\times$ GH200 (120GB) \\
        FSDP & Full Shard, bf16 \\
        \bottomrule
    \end{tabular}
    \label{tab:llm_details:training}
\end{table}
\begin{table}[h]
\centering
\caption{Optimizer hyperparameters (320M). Values marked with $^\dagger$ were tuned via sweep. For Muon, Dion, and Ada-Orth-Dion, the scalar-parameter group (e.g., embeddings, biases, layer norms) is optimized with AdamW-style updates whose hyperparameters are reused from the best AdamW configuration.}
\begin{tabular}{lcccc}
\toprule
\textbf{Hyperparameter} & \textbf{AdamW} & \textbf{Muon} & \textbf{Dion} & \textbf{Ada-Orth-Dion} \\
\midrule
Learning rate$^\dagger$ & 0.012 & 0.016 & 0.012 & 0.012 \\
Output head LR scaling$^\dagger$ & --- & False & False & False \\
$\beta_1$ & 0.95$^\dagger$ & 0.9 & 0.9 & 0.9 \\
$\beta_2$ & 0.95 & 0.95 & 0.95 & 0.95 \\
$\epsilon$ & 1e$-$8 & 1e$-$8 & 1e$-$8 & 1e$-$8 \\
Weight decay & 0.1 & 0.1 & 0.1 & 0.1 \\
Momentum factor & --- & 0.95 & --- & --- \\
$\gamma$ & --- & --- & --- & 1.425 \\
\midrule
Scalar LR & --- & 0.016 & 0.012 & 0.012 \\
Scalar $\beta_1$ & --- & 0.95 & 0.95 & 0.95 \\
Scalar $\beta_2$ & --- & 0.95 & 0.95 & 0.95 \\
Scalar $\epsilon$ & --- & 1e$-$8 & 1e$-$8 & 1e$-$8 \\
Scalar weight decay & --- & 0.1 & 0.1 & 0.1 \\
\bottomrule
\end{tabular}
\label{tab:llm_details:optimizer}
\end{table}
In this paper, we used torchtitan stack to train a Llama 3 architecture model, using C4 data, on the Miyabi-G supercomputer at JCAHPC, with each run using 8 GH200 (120GB) GPUs distributed across 8 nodes interconnected via InfiniBand NDR200.
The number of tokens (3.2B) for the training was calculated based on the Chinchilla optimal heuristic.
For Dion, we used the microsoft implementation~\cite{ahn2025dion}, and our proposed Dion variants extended that implementation.
For AdamW, we swept the joint grid of learning rate $\{0.004, 0.008, 0.012, 0.016\}$ and $\beta_1$ $\{0.9, 0.925, 0.95, 0.975\}$. For Muon and Dion, we swept the joint grid of learning rate $\{0.004, 0.008, 0.012, 0.016\}$ and output head LR scaling $\{\text{True}, \text{False}\}$, where the latter controls whether a shape-dependent rescaling is applied to the learning rate of the final output projection (whose fan-out, the vocabulary size, is much larger than that of the hidden matrices). We reused Dion's best-performing configuration for our proposed methods without further tuning.
For Ada-Orth-Dion, we set the maximum rank at $r_{max}=384$ and initialized the rank at $r_{max}$ to conduct comparison with Dion at $r=384$.
Table~\ref{tab:llm_details:model}, Table~\ref{tab:llm_details:training}, and Table~\ref{tab:llm_details:optimizer} list the model architecture, training hyperparameters, and optimizer hyperparameters used for the experiments, respectively.
A single 320M run (6{,}100 steps, 3.2B tokens) takes 58--76 minutes on 8\,$\times$\,GH200 depending on the optimizer, with peak VRAM of ${\sim}40$\,GB/GPU under FSDP full-shard.

\subsection{Wide-and-Shallow Large-Scale Wall-Clock Study}
\label{app:llm_details:wide}

\begin{table}[h]
    \centering
    \caption{\texttt{w16k\_l4} architecture.}
    \begin{tabular}{lc}
        \toprule
        \textbf{Hyperparameter} & \textbf{Value} \\
        \midrule
        Hidden dimension & 16{,}384 \\
        Intermediate dimension (FFN) & 43{,}776 \\
        Number of layers & 4 \\
        Number of attention heads & 64 \\
        Number of KV heads & 64 \\
        Head dimension & 256 \\
        Sequence length & 2048 \\
        Vocabulary size & 128{,}256 \\
        Total parameters & $\sim$17.1B \\
        \bottomrule
    \end{tabular}
    \label{tab:llm_details:wide_model}
\end{table}

\begin{table}[h]
    \centering
    \caption{Training hyperparameters for the \texttt{w16k\_l4} wall-clock study.}
    \begin{tabular}{lc}
        \toprule
        \textbf{Hyperparameter} & \textbf{Value} \\
        \midrule
        Batch size & 8 \\
        Sequence length & 2048 \\
        Measurement steps & 50 (mean of steps 2--50) \\
        GPUs & 8 $\times$ A40 (44.45GB) \\
        FSDP & Full Shard, bf16 \\
        \bottomrule
    \end{tabular}
    \label{tab:llm_details:wide_training}
\end{table}

To examine wall-clock behavior in the bandwidth-bound regime, we additionally measured per-step time on a wide-and-shallow Llama~3 architecture (referred to as \texttt{w16k\_l4}) with approximately 17.1B parameters (Figure~\ref{fig:fig1_wallclock}).
Table~\ref{tab:llm_details:wide_model} and Table~\ref{tab:llm_details:wide_training} list the model architecture and training used for the wall-clock study.
Each 50-step measurement run takes approximately 40\,minutes on 8\,$\times$\,A40 (48\,GB GDDR6 per GPU, ${\sim}44.5$\,GB usable after framework overhead), with peak VRAM of ${\sim}43$\,GB/GPU under FSDP full-shard.
Aside from the rank fractions set to different values, we reused the best-performing configurations from the 320M sweep for optimizer hyperparameters.
The model architecture is intentionally wide and shallow so that the per-block matrices (e.g., FFN matrices of shape $43{,}776 \times 16{,}384$) are large enough for NCCL collectives to transition from the latency-bound regime to the bandwidth-bound regime. 
The 8\,$\times$\,A40 hardware was chosen for the same reason: A40s have substantially lower inter-GPU bandwidth than current-generation training accelerators (e.g., H100s with NVLink), so the bandwidth-bound regime dominates at smaller collective payload sizes than it would on a higher-end system, making the wall-clock differences between communication-light and communication-heavy optimizers more visible.
In this regime, the reduced communication volume of low-rank Dion-family updates relative to Muon's full-rank Newton--Schulz updates translates directly into wall-clock savings.

For Ada-Orth-Dion, we executed the full adaptive rank algorithm, including the per-block effective rank computation and bookkeeping, but overrode the resulting rank assignment so that the actual working rank was always held fixed at $\rho = 0.93 \times 0.25 = 0.2325$, i.e., 93\% of Dion's fixed rank fraction, regardless of the computed effective rank or the schedule's shrink/grow decision. 
The 93\% factor was chosen to match the average effective rank that Ada-Orth-Dion's adaptive schedule settled at in the 320M experiment, so the wide-model setting reflects its empirically observed steady-state. Since the run still goes through the full Ada-Orth-Dion code path (sort-and-pack batching, per-block adaptive bookkeeping, Cholesky-QR), any algorithmic and implementation overhead inherent to Ada-Orth-Dion is reflected in the measurement, while holding the rank constant isolates whether the rank reduction enabled by the adaptive mechanism is sufficient to recover this overhead and yield a net wall-clock improvement over fixed-rank Orth-Dion or Dion.
\newpage
\section*{NeurIPS Paper Checklist}

\begin{enumerate}

\item {\bf Claims}
    \item[] Question: Do the main claims made in the abstract and introduction accurately reflect the paper's contributions and scope?
    \item[] Answer: \answerYes{}
    \item[] Justification: The introduction (Section~\ref{sec:intro}) explicitly enumerates the four contributions (geometric diagnosis of the $\sqrt{r}$ Ky~Fan dual-norm inflation in Dion; Orth-Dion as a one-line ColNorm$\to$QR fix; convergence theory under non-Euclidean smoothness with self-consistent residual control and amortized contraction; and Ada-Orth-Dion at LLM scale) and the scope is constrained to the empirical regimes actually studied (Llama~3 320M pretraining on C4, GPT-2 45M on WikiText-103 for the $\beta$-sweep, and a Llama~3 17.1B wide-and-shallow wall-clock measurement).

\item {\bf Limitations}
    \item[] Question: Does the paper discuss the limitations of the work performed by the authors?
    \item[] Answer: \answerYes{}
    \item[] Justification: Section~\ref{sec:conclusion} (``Limitations'' paragraph) states that the convergence analysis assumes deterministic gradients (the stochastic extension requires bounding mini-batch noise/subspace-tracking interaction) and that the adaptive-rank policy currently uses a heuristic effective-rank estimator rather than a provably-tight estimate of the contraction-condition critical rank. Section~\ref{sec:exp_beta} additionally notes that the $\beta$-sweep is a single-seed study, and Section~\ref{sec:exp_wallclock_17b} notes that the 17.1B wall-clock study fixes Ada-Orth-Dion's working rank to its 320M steady-state value rather than running the adaptive schedule.

\item {\bf Theory assumptions and proofs}
    \item[] Question: For each theoretical result, does the paper provide the full set of assumptions and a complete (and correct) proof?
    \item[] Answer: \answerYes{}
    \item[] Justification: Assumptions A, B$'$, C$'$ are stated in Section~\ref{sec:geometry} and elaborated in Appendix~\ref{app:assumptions}. The one-step descent inequality is proved in Appendix~\ref{app:one_step}; the Orth-Dion structural lemmas (Lemmas~\ref{lem:partial_isometry}--\ref{lem:same_tracking}, Proposition~\ref{prop:colnorm_sqrt_r}) in Appendix~\ref{app:orth_dion_proofs}; the self-consistent residual fixed point in Appendix~\ref{app:self_consistent_proof}; the amortized-contraction tracking-error bound in Appendix~\ref{app:amortized_proof}; the main convergence theorem (Theorem~\ref{thm:orth_dion}) in Appendix~\ref{app:main_proof}; and the $\beta$-regime derivations in Appendix~\ref{app:beta_details}.

\item {\bf Experimental result reproducibility}
    \item[] Question: Does the paper fully disclose all the information needed to reproduce the main experimental results of the paper to the extent that it affects the main claims and/or conclusions of the paper (regardless of whether the code and data are provided or not)?
    \item[] Answer: \answerYes{}
    \item[] Justification: Section~\ref{sec:exp_llm} specifies the 320M experimental setup (model, dataset, token budget, hardware, FSDP configuration, baselines, schedule, $\beta$). Appendix~\ref{app:llm_details} lists the full Llama~3 320M architecture (Table~\ref{tab:llm_details:model}), training hyperparameters (Table~\ref{tab:llm_details:training}), per-optimizer hyperparameters and the LR / output-head-scaling sweep grid (Table~\ref{tab:llm_details:optimizer}), the Dion implementation reused, and the upstream Microsoft Dion implementation our variants extend. Appendix~\ref{app:llm_details:wide} mirrors this for the 17.1B wall-clock study (architecture, sequence length, batch size, FSDP/precision, hardware rationale, and the rank-pinning protocol used for Ada-Orth-Dion). The $\beta$-sweep configuration (model, dataset, $r$, step count, schedule) is given in Section~\ref{sec:exp_beta} and Table~\ref{tab:beta_sweep}. Algorithms~\ref{alg:orth_dion} (Orth-Dion) and~\ref{alg:rank_policy} (adaptive-rank policy, Appendix~\ref{app:ada_dion_algo}) specify the methods at the line-of-pseudocode level.

\item {\bf Open access to data and code}
    \item[] Question: Does the paper provide open access to the data and code, with sufficient instructions to faithfully reproduce the main experimental results, as described in supplemental material?
    \item[] Answer: \answerYes{}
    \item[] Justification: An anonymized code release accompanying the submission contains the Orth-Dion / Ada-Orth-Dion implementation (built on top of the public Microsoft Dion implementation~\cite{ahn2025dion} and the torchtitan training stack), launch scripts, configuration files for the 320M, 17.1B, and GPT-2 45M experiments, and the analysis scripts that produce the $\nu_t$ measurements and figures. The dataset (C4) and tokenizer are public; the training corpus is loaded directly from the standard release. We do not redistribute pretrained weights.

\item {\bf Experimental setting/details}
    \item[] Question: Does the paper specify all the training and test details (e.g., data splits, hyperparameters, how they were chosen, type of optimizer, etc.) necessary to understand the results?
    \item[] Answer: \answerYes{}
    \item[] Justification: Section~\ref{sec:experiments} (subsections~\ref{sec:exp_llm}, \ref{sec:exp_sqrt_r}, \ref{sec:exp_wallclock_17b}, \ref{sec:exp_beta}) reports the per-experiment training/eval configuration. Appendix~\ref{app:llm_details} provides the full set of hyperparameters and the sweep grids used for each optimizer (LR, $\beta_1$, output-head LR scaling), and explicitly states that the proposed methods reuse Dion's best-tuned configuration without further tuning so that Dion vs.\ Orth-Dion comparisons are hyperparameter-controlled. Validation is on the C4 validation split (320M) and the standard WikiText-103 split (GPT-2 45M).

\item {\bf Experiment statistical significance}
    \item[] Question: Does the paper report error bars suitably and correctly defined or other appropriate information about the statistical significance of the experiments?
    \item[] Answer: \answerYes{}
    \item[] Justification: Table~\ref{tab:llm_main} and Figure~\ref{fig:fig1_val_loss} report mean $\pm$ 2 std over $3$ seeds for the 320M Llama~3 results. Figure~\ref{fig:nu_bar} reports mean over $3$ Dion seeds with a $\pm$1~std band, and Figure~\ref{fig:nu_hist} pools over (layer, step, seed). Figure~\ref{fig:fig1_wallclock} (and the 17.1B wall-clock plots in Appendix~\ref{app:llm_details:wide}) report mean $\pm$ 2 std across the measured per-step times within each run. The GPT-2 45M $\beta$-sweep (Table~\ref{tab:beta_sweep}) is single-seed due to compute constraints; this is acknowledged as a limitation, and the $\beta$-vs.-$\cR_t^*$ trend is monotone over a wide range of $\beta$, which we report as evidence of a genuine effect rather than seed noise.

\item {\bf Experiments compute resources}
    \item[] Question: For each experiment, does the paper provide sufficient information on the computer resources (type of compute workers, memory, time of execution) needed to reproduce the experiments?
    \item[] Answer: \answerYes{}
    \item[] Justification: Appendix~\ref{app:llm_details:320m} states that each 320M run (6{,}100 steps, 3.2B tokens) takes 58--76~minutes on $8\times$GH200 (120GB) interconnected via InfiniBand NDR200, with peak ${\sim}40$~GB/GPU under FSDP full-shard. Appendix~\ref{app:llm_details:wide} states that each 50-step 17.1B measurement run takes ${\sim}40$~minutes on $8\times$A40 (44.45~GB usable) under FSDP full-shard, with peak ${\sim}43$~GB/GPU, and explains the choice of A40 hardware (lower inter-GPU bandwidth) for the bandwidth-bound regime study. The GPT-2 45M $\beta$-sweep configuration ($r{=}64$, 30k steps) is given in Section~\ref{sec:exp_beta}.

\item {\bf Code of ethics}
    \item[] Question: Does the research conducted in the paper conform, in every respect, with the NeurIPS Code of Ethics?
    \item[] Answer: \answerYes{}
    \item[] Justification: This work studies optimizer geometry for distributed pretraining and does not involve human subjects, sensitive data, deception, dual-use models, or deployment risks beyond those already present in standard language-model pretraining.

\item {\bf Broader impacts}
    \item[] Question: Does the paper discuss both potential positive societal impacts and negative societal impacts of the work performed?
    \item[] Answer: \answerNA{}
    \item[] Justification: The contribution is a low-rank distributed optimizer for matrix-shaped parameters. Like other algorithmic improvements to pretraining throughput, any societal impact is mediated entirely by the downstream models that practitioners may train more efficiently as a result, and is not specific to this method.

\item {\bf Safeguards}
    \item[] Question: Does the paper describe safeguards that have been put in place for responsible release of data or models that have a high risk for misuse?
    \item[] Answer: \answerNA{}
    \item[] Justification: The paper does not release pretrained weights or new datasets. The released code is an optimizer implementation and supporting analysis scripts.

\item {\bf Licenses for existing assets}
    \item[] Question: Are the creators or original owners of assets used in the paper properly credited and are the license and terms of use explicitly mentioned and properly respected?
    \item[] Answer: \answerYes{}
    \item[] Justification: All upstream assets are cited at point of use: the C4 dataset~\citep{raffel2020t5}, the Llama~3 architecture~\citep{grattafiori2024llama}, the Chinchilla token-budget rule~\citep{hoffmann2022chinchilla}, FSDP~\citep{zhao2023pytorch}, Muon~\citep{jordan2024muon}, and the Dion algorithm and Microsoft implementation we extend~\citep{ahn2025dion}, along with related baselines (PowerSGD~\citep{vogels2019powersgd}, GaLore~\citep{zhao2024galore}, ReLoRA~\citep{lialin2023relora}, Shampoo~\citep{gupta2018shampoo}, SOAP~\citep{vyas2024soap}). The torchtitan training stack used for the 320M and 17.1B runs is acknowledged in Appendix~\ref{app:llm_details:320m}. We use these assets within their stated open-source licenses.

\item {\bf New assets}
    \item[] Question: Are new assets introduced in the paper well documented and is the documentation provided alongside the assets?
    \item[] Answer: \answerYes{}
    \item[] Justification: The accompanying anonymized code release contains the Orth-Dion and Ada-Orth-Dion optimizer implementations together with launch scripts, configuration files, and the analysis pipeline that generates the reported figures and tables, with a README documenting environment requirements and entry points.

\item {\bf Crowdsourcing and research with human subjects}
    \item[] Question: For crowdsourcing experiments and research with human subjects, does the paper include the full text of instructions given to participants?
    \item[] Answer: \answerNA{}
    \item[] Justification: The paper does not involve crowdsourcing or human subjects.

\item {\bf Institutional review board (IRB) approvals or equivalent for research with human subjects}
    \item[] Question: Does the paper describe potential risks incurred by study participants?
    \item[] Answer: \answerNA{}
    \item[] Justification: The paper does not involve human subjects.

\item {\bf Declaration of LLM usage}
    \item[] Question: Does the paper describe the usage of LLMs if it is an important, original, or non-standard component of the core methods in this research?
    \item[] Answer: \answerNA{}
    \item[] Justification: LLMs (Llama~3, GPT-2) appear only as the training targets used to evaluate the proposed optimizer; they are not a component of the method, and no LLM is part of the algorithmic pipeline being analyzed.

\end{enumerate}

\end{document}